%% file: main.tex
\documentclass{article}
\input{common}

\begin{document}

\title{Strategyproof Learning: \\ Collecting Trustworthy User-Generated Datasets}

\author[1]{Sadegh Farhadkhani}
\author[1]{Rachid Guerraoui}
\author[1]{Lê-Nguyên Hoang}
\author[1]{Leo Serena}
\affil[1]{IC, EPFL, Switzerland}

\maketitle

\input{abstract}

\input{introduction}

\input{related_work}
\input{model}

\input{strategyproofness}

\input{asymptotic}

\input{experiment}

\input{conclusion}

\input{acknowledgment}

\bibliographystyle{alpha}
\bibliography{references}

\newpage
\begin{center}
{\huge Appendix}
\end{center}

\appendix
\input{existence_uniqueness_proof}

\input{strategyproofness_proof}

\input{asymptotic_proof}

\end{document}

%% file: common.tex
\usepackage{xcolor}

\usepackage[utf8]{inputenc}
\usepackage{amsmath,amssymb,amsthm}
\usepackage{bm,bbm,xfrac,url,hyperref,authblk}
\usepackage{fullpage}
\usepackage{dsfont}
\usepackage{caption}

\usepackage{todonotes}

\makeatletter
\newtheorem*{rep@theorem}{\rep@title}
\newcommand{\newreptheorem}[2]{%
\newenvironment{rep#1}[1]{%
 \def\rep@title{#2 \ref{##1}}%
 \begin{rep@theorem}}%
 {\end{rep@theorem}}}
\makeatother

\newtheorem{theorem}{Theorem}
\newreptheorem{theorem}{Theorem}
\newtheorem{lemma}{Lemma}
\newreptheorem{lemma}{Lemma}
\newtheorem{proposition}{Proposition}
\newreptheorem{proposition}{Proposition}

\newcommand{\mathsep}{,~}
\newcommand{\st}{\,\middle|\,}
\newcommand{\set}[1]{\left\lbrace #1 \right\rbrace}
\newcommand{\card}[1]{\left\lvert{#1}\right\rvert}
\newcommand{\absv}[1]{\card{#1}}
\newcommand{\norm}[2]{\left\lVert{#1}\right\rVert_{#2}}

\newcommand{\setR}{\mathbb R}

\newcommand{\probability}[1]{\mathbb P\left[#1\right]}

\newcommand{\expectVariable}[2]{\mathbb E_{#1}\left[#2\right]}

\DeclareMathOperator*{\argmin}{arg\,min}

\newtheorem{definition}{Definition}

\newcommand{\model}{\theta}
\newcommand{\modelsub}[1]{\model_{#1}}
\newcommand{\modelfamily}{\vec \theta}

\newcommand{\targetmodel}{\model^{\dagger}_{\strategicuser}}

\newcommand{\attackmodel}{\model^{\spadesuit}_{\strategicuser}}
\newcommand{\globalmodel}{\rho}

\newcommand{\attackdata}{\data{\strategicuser}^\spadesuit}

\newcommand{\optmodelsub}[1]{\model_{#1}^*}
\newcommand{\optmodelfamily}{\vec{\model}^*}

\newcommand{\optglobalmodel}{\globalmodel^*}
\newcommand{\defaultglobalmodel}{\globalmodel^{\infty}}

\newcommand{\dimension}{d}

\newcommand{\honestmodel}{\model^\dagger}
\newcommand{\honestmodelsub}[1]{\model_{#1}^\dagger}

\newcommand{\query}[1]{\mathcal Q_{#1}}

\newcommand{\answer}[1]{\mathcal A_{#1}}

\newcommand{\varx}{x}
\newcommand{\varxsub}[1]{\varx_{#1}}

\newcommand{\vary}{y}
\newcommand{\varysub}[1]{\vary_{#1}}

\newcommand{\varz}{z}

\newcommand{\licchaviweight}{w}
\newcommand{\regweightsub}[1]{\lambda_{#1}}
\newcommand{\regweightglobal}[1]{\mu}
\newcommand{\reglocalweight}[1]{\nu}

\newcommand{\regnormsub}[1]{q_{#1}}
\newcommand{\regnormcommon}[1]{\regnormsub{0}}
\newcommand{\regpowersub}[1]{p_{#1}}
\newcommand{\regpowercommon}[1]{\regpowersub{0}}

\newcommand{\globalmodelnorm}[1]{\mathcal N_{0}}

\newcommand{\Probability}{P}

\newcommand{\modelbound}{\mathcal K}
\newcommand{\event}{\mathcal E}

\newcommand{\USERINPUT}[1]{\mathcal{I}_{#1}}

\newcommand{\coordinate}{i}
\newcommand{\data}[1]{\mathcal D_{#1}}
\newcommand{\datafamily}[1]{\vec{\mathcal D}_{#1}}

\newcommand{\user}{n}
\newcommand{\USER}{N}

\newcommand{\strategicuser}{s}
\newcommand{\targetuser}{t}

\newcommand{\orthogonalmatrix}{Q}
\newcommand{\orthogonalmatrixfamily}{\vec Q}
\newcommand{\noise}[1]{\xi_{#1}}
\newcommand{\unitvector}[1]{{\bf u}_{#1}}

\newcommand{\strategyproofbound}{\alpha}

\newcommand{\diagonal}[1]{\Delta_{#1}}
\newcommand{\utility}[1]{u}

\newcommand{\hessian}[1]{H_{#1}}
\newcommand{\eigenvalue}[1]{\lambda_{#1}}
\newcommand{\sdpmatrix}{S}

\newcommand{\localloss}[1]{\mathcal{L}_{#1}}
\newcommand{\sign}{\textsc{sgn}}

\newcommand{\reducedloss}[1]{\mathcal R_{#1}}
\newcommand{\expectedreducedloss}{\bar{\mathcal R}}

\newcommand{\spectrum}{\textsc{Sp}}

\newcommand{\licchavi}{\textsc{Licchavi}}
\newcommand{\licch}{\textsc{Lch}}

\newcommand{\alg}{\textsc{Alg}}
\newcommand{\achievableset}{\textsc{AchSet}}
\newcommand{\skewness}{\textsc{Skew}}
\newcommand{\crookedness}{\textsc{Crooked}}
\newcommand{\diag}{\textsc{Diag}}
\newcommand{\face}{\textsc{Face}}
\newcommand{\edge}{\textsc{Edge}}

\newcommand{\huber}{\textsc{Hb}}
\newcommand{\huberparameter}[1]{\delta_{#1}}
\newcommand{\huberconstant}{\delta_c}
\newcommand{\huberfunction}[1]{\sqrt{1+#1}}
\newcommand{\scalerhuber}{\textsc{Hb}}

\newcommand{\canonicalbasis}{\mathcal E}
\newcommand{\canonicalvector}[1]{{\bf e}^{#1}}
\newcommand{\datadistribution}{\tilde{\mathcal D}}

\newcommand{\errorfunction}[1]{\textsc{Err}_{#1}}
\newcommand{\modelreducedloss}[1]{\mathcal S_{#1}}

%% file: abstract.tex
\begin{abstract}
  We prove in this paper that, perhaps surprisingly, incentivizing data misreporting is not a fatality.
  By leveraging a careful design of the loss function, 
  we propose \licchavi{}, 
  a global and personalized learning framework with provable \emph{strategyproofness} guarantees. 
  Essentially, 
  we prove that no user can gain much by replying to \licchavi{}'s queries with answers that deviate from their true preferences.
  Interestingly, \licchavi{} also promotes the desirable ``one person, one unit-force vote'' fairness principle. 
  Furthermore, our empirical evaluation of its performance 
  showcases \licchavi{}'s real-world applicability.
  We believe that our results are critical for the safety of any learning scheme that leverages user-generated data.
\end{abstract}

%% file: introduction.tex
\section{Introduction}

Today's large-scale algorithms,
designed for autocompletion~\cite{LehmannB21}, conversational~\cite{ShumHL18} and recommendation~\cite{IeJWNAWCCB19} applications,
exploit the data generated from the activities of a large number of users~\cite{SmithSPKCL13,WangPNSMHLB19,WangSMHLB19}
to construct both \emph{global} and \emph{personalized} models~\cite{RicciRS11,FallahMO20,HanzelyHHR20}.

However, the fact that \emph{strategic} users may provide \emph{untrustworthy} data challenges the classical theory of learning, which generally regards as desirable to fit available data, and to generalize them for future applications~\cite{Valiant84}.
In applications such as content recommendation, %
 activists, companies and politicians usually have strong incentives to promote certain views, products or ideologies~\cite{Hoang20,HoangFE21}.
After all, two YouTube views out of three result from algorithmic recommendations~\cite{cnet18}.
Quite naturally, this has led to vast amounts of fabricated activities to bias algorithms~\cite{bradshaw19,neudert2019},
through ``fake reviewing''~\cite{WuNWW20}, ``astroturfing''~\cite{ZerbackTK21} or automated harassment~\cite{tekfog}.
In fact, Facebook reportedly removed 15~billion fake accounts within two years~\cite{facebook_fake_accounts}.
This raises serious concerns, especially given today's ``stochastic parrots''~\cite{BenderGMS21}: today's language models incentivize anti-vaccine groups to heavily pollute textual datasets with claims like ``vaccines kill'', including through fake accounts, as autocompletion, conversational and recommendation algorithms trained on such data will more likely spread this view~\cite{McGuffieNewhouse20}.

Arguably, in large-scale environments that naturally attract a large number of malicious entities, like social medias, any data that is not cryptographically signed by authentic trustworthy entities should not be trusted.
In other words, a necessary condition for the safety of learning algorithms is to train them solely on {\it signed} data, that is, data that provably come from a known user. Yet,  this is clearly not a sufficient condition for safety:
even signed data cannot be wholeheartedly trusted.
After all, even authentic users usually have preferences over what ought to be recommended to others, and thus have incentives to behave strategically.

Unfortunately, today's state-of-the-art algorithms strongly incentivize and are extremely vulnerable to such strategic manipulations.
In fact, it was shown by \cite{equivalence_data_gradient} that classical personalized federated learning algorithms like \cite{DinhTN20,HanzelyHHR20} can be arbitrarily manipulable by a single strategic user, through the injection of a surprisingly small amount of poisonous data.
In particular, such a user would be incentivized to construct an attack model, and to provide data labeled with this attack model rather than with the user's preferred model.
Assuming most users behaving strategically, the data thereby collected would inevitably be hopelessly untrustworthy:
any algorithm trained with such data could be dangerously manipulated and weaponized by malicious data providers.

In this paper, we ask whether an algorithm can achieve performant (personalized) learning, while incentivizing users' honest data generation and reporting.
In the parlance of social choice theory, such an algorithm is called \emph{strategyproof}.

To pose and address the question in a precise manner, we
propose a new and rigorous definition of strategyproofness in the context of learning from user-generated data.
We then
introduce \licchavi{}\footnote{The Licchavis were a clan in Ancient India, and are credited for their early form of proto-democracy.}, a relatively simple, yet general, learning framework, based on a careful design of the underlying loss function.
We assume that users generate data by labeling them.
Equivalently, this can be regarded as users being given queries, and providing answers to these queries.
Such answers can be given honestly, using the user's implicit preferred model, or can be given strategically to bias the global model, or other users' learned models.
\licchavi{} then leverages users' data to perform both global and personalized learning,
by penalizing the discrepancy between the global model and users' models.
This essentially captures the trade-off, for each user, between fitting other users' data and model personalization.
Critically, we use a \emph{coordinate-wise pseudo-Huber} penalization, which allows to derive strategyproofness guarantees.

\licchavi{} also has applications in high-dimensional voting, e.g. to determine the parameters of a content moderation algorithm on a social network.
In the same vein as~\cite{NoothigattuGADR18,LeeKKK+19,FreedmanBSDC20}, \licchavi{} would query the voter, collect the voter's answers and then use machine learning to model the voter's preferences.
In practice, however, especially in high dimensions, each voter often provides an insufficient amount of data.
This prevents the model from reliably learning their preferences.
\licchavi{} allows improving the sample complexity by leveraging other voters' inputs to better learn a voter's preferences.
More importantly, the global model learned by \licchavi{} can then be regarded as the output of the high-dimensional voting.

Interestingly, in addition to providing strategyproofness guarantees, the use of such coordinate-wise pseudo-Huber penalizations also implies an intuitively appealing fairness principle that \licchavi{} satisfies.
Basically, we show that \licchavi{} essentially fits the appealing fairness principle ``one voter, (at most) one unit force''~\cite{geometric_median}, assuming that this force is measured by the $\ell_\infty$-norm, while also accounting for the uncertainty \licchavi{} inevitably has on a user's preferred model, when the user does not provide sufficiently many data.
All in all, this makes \licchavi{} a very promising tool for scalable algorithmic governance, especially in controversial contexts where users' preferences are expected to greatly diverge.

\paragraph{Contributions.}
Our main contribution is to introduce \licchavi{} and to analyze its  {\it strategyproofness}, i.e., whether it is in each user's interest to answer queries honestly.
We first prove that, unfortunately, assuming that each user wants to minimize the Euclidean distance between a target user's model and their preferred model,  \licchavi{} cannot be guaranteed to always be $\strategyproofbound$-strategyproof.
Fortunately, we also prove that, for gradient PAC* coordinate-wise separable local losses, \licchavi{} is guaranteed to be strategyproof.
We also discuss how to leverage this result to tune \licchavi{} to obtain approximate strategyproofness in the general case, when local losses are not coordinate-wise separable.

Our second result is that, even without this tuning, in the asymptotic case of a large number of voters, \licchavi{} is $\strategyproofbound$-strategyproof, for a value of $\strategyproofbound$ that we explicitly compute based on the distribution of voters' preferred models.
In short, we argue that the study of the asymptotic strategyproofness of \licchavi{} can be reduced to the study of a related strategyproofness problem.
We then go on proving that, for this related problem, \licchavi{} is $\strategyproofbound$-strategyproof.
This result constitutes a fair argument for why strategic users will not have strong incentives to provide fabricated rather than honest data, in the general case.

Our paper also shows how easy \licchavi{} is to deploy for practical machine learning tasks.
We do so by considering the case of the personalized federated fine tuning of language models on a set of tweets published on Twitter.
Our empirical evaluation\footnote{The code, the dataset, and the instructions for reproducibility can be found \href{https://drive.google.com/file/d/151ai2290oYYWqQe3LZgRzhzRwo3GMBh8/view?usp=sharing} {\bf here.}} conveys the fact that \licchavi{} provides good performances, at least compared to classical variants~\cite{DinhTN20,HanzelyHHR20}.

%% file: related_work.tex
 \paragraph{Related work.} There is a large body of work 
 on the strategyproofness of learning problems, including regression \cite{ChenPPS18, DEKEL2010759,PEROTE2004153,ben17}, classification \cite{MEIR2012123, chen2020, Meir11, hardt16}, statistical estimation \cite{yang15}, and clustering \cite{Perote04}. 
 The goal has been mainly to train a \emph{single} model that 
 incentivizes the honesty of users who aim to bias the model in their favor (e.g., by pulling the regression model towards their own desired points or achieving a classifier correctly their own labels). 
 But none of these papers studies the strategyproofness of a general \emph{global and personalized learning} framework.
 
 In the case of linear regression, 
 \cite{ChenPPS18} and \cite{PEROTE2004153} assume that each user can only provide a single data point.
 Unfortunately, this greatly restricts the users' ability to contribute to the learning model.
 Whilst \cite{DEKEL2010759} allows users to provide multiple data points, they either require payments, which might not be possible (e.g., due to ethical reasons), or they restrict the model to one dimension or a constant function in $\setR^d$. 
 \licchavi{}, in contrast, does not make use of any payment, nor does does it restrict the dimension of the model, and yet enables users to contribute large datasets.

Note that other desirable properties of coordinate-wise regularizations in general (typically $\ell_1$ regularization) have been previously observed,
both in terms of generalization~\cite{tibshirani1996,WANG2013135,ShiKAZ22}, robustness~\cite{XuCM08,Ding18,PesmeF20} and strategyproofness~\cite{goel20,DEKEL2010759} (in restricted settings).
Here, we show how it can be used to provide strategyproofness guarantees for a very general global and personalized learning scheme.

\paragraph{Structure of the paper.} The rest of the paper is organized as follows. 
Section~\ref{sec:model} introduces 
introduce  \licchavi{}.
Section~\ref{sec:strategyproof} presents our first main contribution, the strategyproofness analysis for the non-asymptotic case.
We also discuss the tuning of \licchavi{} for approximate strategyproofness.
Section~\ref{sec:asymptotic} introduces our second main contribution,  the asymptotic strategyproofness analysis.
Section~\ref{sec:experiment} presents our empirical evaluation of \licchavi{}.
Section~\ref{sec:conclusion} concludes.
Proofs are provided in the Appendix.

%% file: model.tex
\section{\licchavi{}}
\label{sec:model}

We consider a set $[\USER] = \{ 1, \ldots, \USER \}$ of users.
Each user $\user \in [\USER]$ is repeatedly provided with queries $\query{}$ (which they may select themselves), and is asked to provide answers $\answer{}$.
The set of user $\user$'s query-answer pairs $(\query{}, \answer{})$ forms the user's reported dataset $\data{\user}$.
We denote by $\datafamily{} \triangleq (\data{1}, \ldots, \data{\USER})$ the tuple of users' datasets.

Our goal is to perform both \emph{global and personalized learning} (GPL).
Namely, for each user $\user$, we want to recover a model $\modelsub{\user} \in \setR^\dimension$ that fits and generalizes their reported data $\data{\user}$. We let $\modelfamily{} \triangleq \set{\modelsub{1}, \ldots, \modelsub{\USER}}$ denote the the tuple of users' local models.
Additionally, we want to learn a common global model $\globalmodel \in \setR^\dimension$, which may typically be used for community-level decisions, e.g., in the context of content moderation.
This amounts to constructing a GPL algorithm $\alg{} : \datafamily{} \mapsto (\globalmodel^{\alg{}} (\datafamily{}), \modelfamily{}^{\alg{}} (\datafamily{}))$.

To do so, we consider that any user $\user$'s dataset $\data{\user}$ defines a strongly convex and differentiable local loss function $\localloss{} (\modelsub{\user} | \data{\user})$.
We then draw inspiration from \emph{personalized federated learning}~\cite{DinhTN20,HanzelyHHR20,equivalence_data_gradient} to improve sample efficiency, and learn appropriate models even for users whose datasets are very limited, by adding terms that penalize the discrepancies between users' local models $\modelsub{\user}$ and the global model $\globalmodel$.

Now, unfortunately, as shown by \cite{equivalence_data_gradient}, some classical personalized federated learning algorithms like \cite{DinhTN20,HanzelyHHR20} are extremely vulnerable to strategic attacks.
To remedy this vulnerability, we introduce \licchavi{}.
Essentially, \licchavi{} leverages coordinate-wise pseudo-Huber losses~\cite{CharbonnierBAB97,HarlteyZisserman06} to learn a global model.
More precisely, given users' datasets $\datafamily{}$, \licchavi{} outputs a minimum $(\optglobalmodel, \optmodelfamily)$ of the following loss function:
\begin{align}
    \licch{} (\globalmodel, \modelfamily{} | \datafamily{})
    &\triangleq \sum_{\user \in [\USER]} \localloss{} (\modelsub{\user} | \data{\user}) 
    + \licchaviweight\sum_{\user \in [\USER]} \huber_{\frac{\huberconstant}{\huberfunction{\card{\data{\user}}}}} \left(\modelsub{\user} - \globalmodel \right),
\end{align}
where $\huber_{\huberparameter{}} \left(\varz \right)
  \triangleq \sum_{\coordinate \in [d]} \scalerhuber_{\huberparameter{}} (\varz_\coordinate)
  \triangleq \sum_{\coordinate \in [d]} \sqrt{\huberparameter{}^2+ \varz_\coordinate^2}$,
and where $\licchaviweight, \huberconstant > 0$ are hyperparameters of \licchavi{}.
In spirit, \huber{} acts like an $\ell_1$ penalty.
In fact, when $\card{\data{\user}} \rightarrow \infty$, then the \huber{} term converges uniformly to the $\ell_1$ loss.

More importantly, like with $\ell_1$ loss, the pull of each user on the global model
in each direction is bounded by $\licchaviweight$ (this will be formalized by Lemma~\ref{lemma:gradient_reduced_loss}).
This enforces the fairness principle ``one person, (at most) one unit force vote''~\cite{geometric_median}.
This property turns out to be critical for \emph{strategyproofness} (and also implies \emph{robustness}!).
But, interestingly, \huber{} has additional desirable properties.
As opposed to $\ell_1$ loss, $\huber$ is smooth, which makes it easier to optimize and more numerically stable.
Also, the fact that it is closer to a quadratic loss for users with few data points means that such users will act on $\globalmodel$ with a weaker force.
This is consistent with the idea that they ought to be more uncertain about how to pull on $\globalmodel$.
In fact, we chose a typical uncertainty $\frac{\huberconstant}{\huberfunction{\card{\data{\user}}}}$ which decays with the square root of the number of user $\user$'s data, to be consistent with the posterior's standard deviation.
Finally, unlike $\ell_1$ loss, $\huber$ is strictly convex.
Combining all these properties enables us to guarantee that \licchavi{} is well-defined.

\begin{proposition}
\label{prop:existence_uniqueness}
 For any datasets $\datafamily{}$, $\licch{}$ yields a unique minimum,
 which we denote by $\globalmodel^{\licch{}} (\datafamily{})$ and $\modelfamily{}^{\licch{}} (\datafamily{})$.
\end{proposition}

\begin{proof}[Sketch of proof]
The loss is clearly convex overall, and strictly convex in $\modelsub{\user}$.
But given $\modelfamily{}$, it is then strictly convex with respect to $\globalmodel$.
This proves uniqueness.
Moreover, if $\modelsub{\user}$ has a norm too large, then, by strong convexity, the global loss takes values larger than its value at 0.
Thus the minimum must be reached for local models within a compact region.
But then, for $\modelsub{\user}$ in this region, when $\globalmodel$ has a norm too large, the global loss takes values larger than its value at 0.
Hence the minimum must be reached within a bounded region for all models, which proves the existence of a minimum.
The full proof is given in Appendix~\ref{sec:existence_uniqueness_proof}.
\end{proof}

%% file: strategyproofness.tex
\section{Strategyproofness}
\label{sec:strategyproof}

In this section, we study the strategyproofness of \licchavi{}.
We prove that, unfortunately, \licchavi{} provides no general guarantee of $\strategyproofbound$-strategyproofness.
Remarkably, however, we identify a sufficient condition for \licchavi{} to guarantee strategyproofness.
But before presenting our results, we first clearly define strategyproofness, and stress how challenging it is to make any participatory system strategyproof.

\subsection{What is strategyproofness?}
\label{sec:definition_strateyproofness}

Essentially, a participatory system is strategyproof if it incentivizes honest participation.
This means that, in a strategyproof system and in the context of machine learning, it should be in each user's best interests to label data as \emph{they} think the data should be labeled.

\paragraph{Why strategyproofness matters.}
We first stress that strategyproofness is critical for safely learning from user-generated data.
After all, the theory of learning relies on the core principle that \emph{generalizing training data is desirable}.
However, if a learning algorithm strongly incentivizes data misreporting,
perhaps because many users have strong desires or pressures to promote certain products, views or ideologies,
and because dishonesty or misbehaviors strongly favor such outcomes,
then we should expected the algorithm to generalize very misleading, and potentially dangerous, activities.
More generally, \emph{learning algorithms are shaped by their training datasets}.
As a result, their safety strongly depends on the soundness of the data they are trained with.
Strategyproofness is arguably one of the most needed properties to guarantee data soundness,
especially in high-stake environments, e.g., involving information warfare~\cite{lin2019}.

\paragraph{How strategyproofness differs from Byzantine learning.}
Over the last five years, a large body of research~\cite{BlanchardMGS17,MhamdiGR18,baruch19,El-MhamdiGGHR20,KarimireddyHJ21,collaborative_learning,He20} has focused on Byzantine learning, which aims to guarantee the safety of learning despite the presence of participants with arbitrary (potentially maximally malicious) behaviors.
This property is clearly important as well.
After all, especially if the number $\USER$ of users is large, then we should expect the presence of at least a few users with essentially nonsensical activities.

Having said this, we stress that strategyproofness is an orthogonal, complementary and equally important property in practical deployments.
The main reason for this is that strategyproofness considers an arguably more common class of users.
Namely, instead of assuming \emph{arbitrary} or \emph{maximally malicious} behaviors, strategyproofness considers \emph{strategic} users.
Such users are \emph{goal-directed}.
Typically, a strategic user will want the global model to promote their views, or they will want to make other users' models recommend content aligned with the strategic user's preferences.

Crucially, the Byzantine learning literature usually assumes that the vast majority of users behave \emph{honestly}.
This assumption often justifies them in \emph{erasing outliers}.
However, especially in a heterogeneous setting, such as a controversial political debate, erasing outliers can be argued to be unethical, as it amounts to silencing minorities' views.
Perhaps equally importantly, the honest majority assumption also dangerously fails, if \emph{most} users behave strategically.
If so, then the users' reported datasets may be hopelessly dishonest; and generalizing any of it could be highly dangerous.

\paragraph{Strategyproofness is scarce.}
A reader unfamiliar with strategyproofness might feel underwhelmed by the positive results of our paper.
Let us thus stress how rare this property is.
In the 1970s, \cite{gibbard1973} and \cite{satterthwaite1975} independently proved that the only \emph{strategyproof}, \emph{unanimous}\footnote{A vote is unanimous, if, when all users prefer the same alternative and vote honestly, then the vote outputs this unanimously preferred alternative.} and \emph{deterministic} voting algorithm is \emph{dictatorship}.
Later, \cite{gibbard1978} added that the only \emph{strategyproof}, \emph{unanimous} and \emph{neutral}\footnote{A vote is neutral if the alternatives in contention in the vote play a symmetric role.} voting algorithm is \emph{random dictatorship}.
More positive results can be obtained by assuming additional structures on participants' preferences; but even then, they are restrictive.
For instance, \cite{KIM198429} proved that, in dimension 2 and assuming users want the output vector to be as close as possible (in Euclidean norm) to their preferred vector, then the only \emph{strategyproof}, \emph{anonymous}\footnote{A vote is anonymous if the users play a symmetric role.} and \emph{continuous} voting algorithm is the (generalized) coordinate-wise median.
As a fourth example, previous results on strategyproof linear regression by \cite{ChenPPS18} and \cite{PEROTE2004153} only address the very restrictive case where each participant can only provide a single data point.
Given this, our positive theorems about the strategyproofness of \licchavi{} should be regarded as major steps forward in strategyproof learning theory.

\paragraph{Formal definition.}
We now formalize strategyproofness.
The focus here will be on the incentives of any single, omniscient and {\it strategic} user $\strategicuser \in [\USER]$, with a preferred model $\targetmodel$.
We consider that the user's honest behavior consists of (randomly) drawing a large number of queries $\query{}$, and to answer them using their preferred model $\targetmodel$.
The precise way of answering the queries depends on the problem (see~\cite{equivalence_data_gradient}).
For instance, for linear regression, an answer could be of the form $\answer{}^\dagger = \query{}^T \targetmodel + \noise{}$, where $\noise{}$ may typically be a zero-mean noise.
The honest dataset $\data{\strategicuser}^\dagger$ would then be the set of pairs $(\query{}, \answer{}^\dagger)$ thereby constructed.

By contrast, when being strategic, user $\strategicuser$ can report any alternative strategic dataset $\data{\strategicuser}^\spadesuit$.
Additionally, user $\strategicuser$ is assumed to know the datasets $\datafamily{-\strategicuser} \triangleq \left( \data{\user} \right)_{\user \neq \strategicuser}$ provided by other users, and can adapt their choice of the strategic dataset $\data{\strategicuser}^\spadesuit$ accordingly.
Importantly, user $\strategicuser$ is assumed to want to bias the learned global model (or a target user $\targetuser$'s local model) towards their preferred model $\targetmodel$.
More precisely, we assume here that the strategic user's goal is to minimize the Euclidean distance\footnote{Appendix~\ref{app:strategyproof_proof} generalizes our results to any norm invariant by coordinate-wise reflections, e.g., any $\ell_p$ norm.}
between $\globalmodel^\alg{}$ and $\targetmodel$, or between $\modelsub{\targetuser}^\alg{}$ and $\targetmodel$.
Depending on where the strategic user's focus is, we then have the two following definitions.

\begin{definition}%
A global learning algorithm \alg{} is global-targeted $\strategyproofbound$-strategyproof if,
for any preferred model $\targetmodel{} \in \setR^\dimension$
and any other users' datasets $\datafamily{-\strategicuser}$,
given any $\varepsilon, \delta > 0$,
there exists $\USERINPUT{}$ such that,
if $\data{\strategicuser}^\dagger$ is a dataset obtained by honestly answering at least $\USERINPUT{}$ random queries with the preferred model $\targetmodel$,
then with probability at least $1-\delta$,
\begin{align}
    &\forall \data{\strategicuser}^{\spadesuit} \mathsep
    \norm{ \globalmodel^{\alg{}} (\data{\strategicuser}^\dagger, \datafamily{-\strategicuser}) - \targetmodel }{2}
    \leq (1+\strategyproofbound) \norm{ \globalmodel^{\alg{}} (\data{\strategicuser}^\spadesuit, \datafamily{-\strategicuser}) - \targetmodel }{2} + \varepsilon.
\end{align}
If the bound holds for $\strategyproofbound = 0$, then we simply say that \alg{} is global-targeted strategyproof.
\end{definition}

\begin{definition}%
A personalized learning algorithm \alg{} is user-targeted $\strategyproofbound$-strategyproof if,
for any preferred model $\targetmodel{} \in \setR^\dimension$,
any other users' datasets $\datafamily{-\strategicuser}$
and any target user $\targetuser \in [\USER]$,
given any $\varepsilon, \delta > 0$,
there exists $\USERINPUT{}$ such that,
if $\data{\strategicuser}^\dagger$ is a dataset obtained by honestly answering at least $\USERINPUT{}$ random queries with the preferred model $\targetmodel$,
then with probability at least $1-\delta$,
\begin{align}
    &\forall \data{\strategicuser}^{\spadesuit} \mathsep
    \norm{ \modelsub{\targetuser}^{\alg{}} (\data{\strategicuser}^\dagger, \datafamily{-\strategicuser}) - \targetmodel }{2}
    \leq (1+\strategyproofbound) \norm{ \modelsub{\targetuser}^{\alg{}} (\data{\strategicuser}^\spadesuit, \datafamily{-\strategicuser}) - \targetmodel }{2} + \varepsilon.
\end{align}
If the bound holds for $\strategyproofbound = 0$, then we simply say that \alg{} is user-targeted strategyproof.
\end{definition}

\subsection{Main results}

We can now state our main results of this section, which consist of both a negative and a positive theorem.

\begin{theorem}
\label{th:negative_strategyproofness}
  For any $\strategyproofbound > 0$, \licchavi{} is neither global-targeted $\strategyproofbound$-strategyproof nor user-targeted $\strategyproofbound$-strategyproof.
\end{theorem}

Theorem~\ref{th:negative_strategyproofness} stresses the need of further assumptions to retrieve any strategyproofness.
Here, we identify sufficient conditions to guarantee \licchavi{}'s strategyproofness.
The first condition was first introduced by~\cite{equivalence_data_gradient},
who proved it to hold for linear and logistic regression under very mild conditions.

\begin{definition}[Gradient-PAC*, from~\cite{equivalence_data_gradient}]
Denote $\event(\data{}, \honestmodel, \USERINPUT{}, A, B, \alpha)$ the event
\begin{align*}
\label{eq:gradient-pac}
    &\forall \model \in \setR^d \mathsep
    \left( \model - \honestmodel \right)^T \nabla \localloss{} \left(\model | \data{} \right) \geq
    A \USERINPUT{} \min \left\lbrace \norm{\model - \honestmodel}{2}, \norm{\model - \honestmodel}{2}^2 \right\rbrace
    - B \USERINPUT{}^\alpha \norm{\model - \honestmodel}{2}.
\end{align*}
The loss $\localloss{}$ is gradient-PAC* if,
for any $\modelbound > 0$,
there exist $A_\modelbound, B_\modelbound >0$ and $\alpha_\modelbound <1$ such that,
for any preferred model $\honestmodel \in \setR^d$ with $\norm{\honestmodel}{2} \leq \modelbound$,
assuming that the dataset $\data{}$ is obtained by answering random queries $\query{}$ with model $\honestmodel$,
$\probability{\event(\data{}, \honestmodel, \USERINPUT{}, A_\modelbound, B_\modelbound, \alpha_\modelbound)} \rightarrow 1$
as $\USERINPUT{} \rightarrow \infty$.
\label{ass:unbiased}
\end{definition}

Intuitively, gradient PAC* guarantees that if a user $\user$ answers sufficiently many queries by using a labeling model $\honestmodelsub{\user}$, then the labeling model $\honestmodelsub{\user}$ is robustly approximately reconstructed by minimizing the local loss.
To guarantee strategyproofness, we also demand that the local loss $\localloss{}$ be coordinate-wise separable, which means that it can be written $\localloss{} (\model | \data{}) = \sum_{\coordinate \in [\dimension]} \localloss{\coordinate} (\model_\coordinate | \data{})$.
Section~\ref{sec:approximate_strategyproofness} will discuss how this assumption can be removed, by tuning \licchavi{} to provide approximate strategyproofness.

\begin{theorem}
\label{th:strategyproof}
    Assume that the local losses are gradient PAC* and coordinate-wise separable.
    Then \licchavi{} is both global and user-targeted strategyproof.
\end{theorem}

Let us now outline the nontrivial proofs of the two main theorems.
Interestingly, we successfully decomposed them into lemmas, each of which uncovers insights about personalized federated learning in general, and about \licchavi{} in particular.
The lemma proofs appear in Appendix~\ref{app:strategyproof_proof}.

\subsection{Reduced losses}

First, we note that the study of global-targeted strategyproofness can be reduced to the analysis of a loss which only depends on the global model.
To do so, given a local dataset $\data{}$, we first define the reduced local loss
\begin{equation}
\label{eq:reduced_loss}
    \reducedloss{} (\globalmodel | \data{}) \triangleq
    \inf_{\model \in \setR^d} \localloss{} (\model | \data{}) + \licchaviweight \huber_{\frac{\huberconstant}{\huberfunction{\card{\data{}}}}} \left(\modelsub{} - \globalmodel  \right).
\end{equation}
Below, we show that this reduced local loss is well-behaved.

\begin{lemma}
\label{lemma:local_optimum}
Equation~\eqref{eq:reduced_loss} yields a unique minimum $\optmodelsub{} (\globalmodel, \data{})$.
\end{lemma}

\begin{lemma}
\label{lemma:gradient_reduced_loss}
$\reducedloss{} (\globalmodel | \data{})$ is convex and differentiable.
Moreover, $\nabla \reducedloss{} = \licchaviweight \nabla \huber_{\frac{\huberconstant}{\huberfunction{\card{\data{}}}}} \left( \globalmodel - \optmodelsub{} (\globalmodel, \data{}) \right)$, and  $\norm{\nabla \reducedloss{}}{\infty} \leq \licchaviweight$.
\end{lemma}

Let $\reducedloss{} (\globalmodel | \datafamily{}) \triangleq \sum_{\user \in [\USER]} \reducedloss{} (\globalmodel | \data{\user})$ and $\reducedloss{} (\globalmodel | \datafamily{-\strategicuser}) \triangleq \sum_{\user \neq \strategicuser} \reducedloss{} (\globalmodel | \data{\user})$ be the sum of (other) users' reduced losses.

\begin{lemma}
\label{lemma:local_global_optimum}
$\globalmodel^{\licch{}} (\datafamily{})$ is the unique minimum of $\reducedloss{} (\globalmodel | \datafamily{})$,
while $\modelsub{\user}^{\licch{}} (\datafamily{}) = \optmodelsub{} (\globalmodel^{\licch{}} (\datafamily{}), \data{\user})$.
\end{lemma}

\subsection{Strong local PAC*}
\label{sec:pac}

Another key step of our proofs is to reduce data reporting strategyproofness to model reporting strategyproofness, for gradient PAC* local losses.
\cite{equivalence_data_gradient} also proved that gradient PAC* implies \emph{local PAC*} learning for a large class of personalized federated learning algorithm.
In this paper, we prove a stronger result for the particular case of \licchavi{}.
Namely, we prove that, under gradient PAC* local losses, \licchavi{} is \emph{strongly} local PAC*.

\begin{definition}
\label{def:strong_pac}
  A GPL algorithm \alg{} is strongly local PAC* if,
  for any user $\user$ and any preferred model $\honestmodelsub{\user}$,
  any $\varepsilon, \delta>0$,
  there exists $\USERINPUT{}$ such that,
  if the user $\user$ provides a dataset $\data{\user}^\dagger$ with $\card{\data{\user}^\dagger} \geq \USERINPUT{}$ answers to random queries given using their preferred models $\honestmodelsub{\user}$,
  then, with probability at least $1-\delta$,
  \begin{equation}
      \forall \datafamily{-\user} \mathsep
      \norm{\modelsub{\user}^{\alg{}} \left( \data{\user}^\dagger, \datafamily{-\user} \right) - \honestmodelsub{\user}}{2} \leq \varepsilon
  \end{equation}
\end{definition}

Importantly, as opposed to local PAC* (introduced in \cite{equivalence_data_gradient}), \emph{strong} local PAC* guarantees the accuracy of the learning of $\honestmodelsub{\user}$ \emph{independently} from other users' data $\datafamily{-\user}$.
This is a very desirable property in practice, as it guarantees that a user with sufficiently many data will never be hacked by a very active malicious user.
Interestingly, this is a property that \licchavi{} guarantees.

\begin{lemma}
\label{lemma:strong_local_pac}
  For gradient PAC* local losses, \licchavi{} is strongly local PAC*.
\end{lemma}

\begin{proof}[Sketch of proof]
  The key insight is that the pseudo-Huber regularization term of \eqref{eq:reduced_loss} has a bounded gradient.
  By contrast, by gradient PAC*, as a user $\user$ with preferred model $\honestmodelsub{\user}$ provides more and more honest data $\data{\user}^\dagger$,
  for any $\modelsub{\user}$ too far from the preferred model $\honestmodelsub{\user}$,
  the negative gradient $-\nabla_{\modelsub{\user}} \localloss{} (\modelsub{\user} | \data{\user}^\dagger)$ of the local loss will point more and more towards $\honestmodelsub{\user}$,
  so that it will eventually outweigh the gradient $\nabla_{\modelsub{\user}} \licchaviweight \huber_{\frac{\huberconstant}{\huberfunction{\card{\data{}}}}} \left(\modelsub{\user} - \globalmodel \right)$ of the pseudo-Huber regularization term,
  no matter what value $\globalmodel$ takes.
  This guarantees that, for any value of $\globalmodel^{\licch{}} (\datafamily{})$,
  the optimum $\modelsub{\user}^\licch{} (\datafamily{}) = \optmodelsub{} (\globalmodel^{\licch{}} (\datafamily{}), \data{\user})$ will be close to $\honestmodelsub{\user}$.
\end{proof}

\subsection{Reduction to model attack}

By (strong) local PAC* and by providing enough data $\data{\strategicuser}^\spadesuit$ labeled with $\attackmodel$,
the strategic user $\strategicuser$ can essentially make \licchavi{} learn the model $\modelsub{\strategicuser}^\licch{} \approx \attackmodel$.
Moreover, by providing enough data, they can make the Huber loss essentially equal to an $\ell_1$ loss.
This prompts us to consider the following modified Licchavi loss
\begin{align}
    &\licch{}_{\strategicuser} (\globalmodel | \attackmodel, \datafamily{-\strategicuser})
    \triangleq \licchaviweight \norm{ \attackmodel - \globalmodel }{1} + \reducedloss{} (\globalmodel | \datafamily{-\strategicuser}).
\end{align}
This loss can be easily shown to yield a unique minimum,
which we denote by $\globalmodel^{\licch{}} (\attackmodel, \datafamily{-\strategicuser})$
and $\modelsub{\user}^{\licch{}} (\attackmodel, \datafamily{-\strategicuser})$ for $\user \neq \strategicuser$.
Define also $\modelsub{\strategicuser}^{\licch{}} (\attackmodel, \datafamily{-\strategicuser}) \triangleq \attackmodel$.
The definition of $\strategyproofbound$-strategyproofness under model attack is then akin to the definitions of Section~\ref{sec:definition_strateyproofness}, but with models instead of data, and without any randomness and approximation, which removes the need of $\varepsilon$ and $\delta$.
Typically, for the case of global-targeted $\strategyproofbound$-strategyproofness, the following must hold:
\begin{align}
    \forall \attackmodel, \targetmodel \mathsep
    &\forall \datafamily{-\strategicuser} \mathsep
    \norm{ \globalmodel^{\alg{}} (\targetmodel, \datafamily{-\strategicuser}) - \targetmodel }{2}
    \leq (1+\strategyproofbound) \norm{ \globalmodel^{\alg{}} (\attackmodel, \datafamily{-\strategicuser}) - \targetmodel }{2}.
\end{align}
We can now adapt the equivalence proven by \cite{equivalence_data_gradient} to the case of \licchavi{}'s strategyproofness.

\begin{lemma}
\label{lemma:equivalence_data_model}
Assuming strong local PAC*,
\licchavi{} is global-targeted $\strategyproofbound$-strategyproof under data attack if and only if it is global-targeted $\strategyproofbound$-strategyproof under model attack.
The equivalence also holds for user-targeted $\strategyproofbound$-strategyproofness.
\end{lemma}

\begin{proof}[Sketch of proof]
  On one hand, any data attack $\attackdata$ yields the same outcome as the attack by model $\attackmodel \triangleq \globalmodel^{\licch{}} (\attackdata, \data{-\strategicuser})$.
  On the other hand, by strong local PAC* (Lemma~\ref{lemma:strong_local_pac}), an attack model $\attackmodel$ yields essentially the same result as the dataset $\attackdata$ obtained by randomly a large number of queries and answering them with model $\attackmodel$.
  The precise analysis, given in Appendix~\ref{app:equivalence}, is however nontrivial.
\end{proof}

In light of the lemma, to prove theorems~\ref{th:negative_strategyproofness} and~\ref{th:strategyproof}, it is sufficient to (dis)prove strategyproofness under model attack.

\subsection{Proof sketch of the negative result}

Unfortunately, in general, no $\strategyproofbound$-strategyproofness guarantee holds for \licchavi{}.

\begin{proof}[Sketch of proof]
  Essentially, we construct a nasty instance for $\dimension = 2$, by designing appropriately the other users' reduced loss $\reducedloss{} (\globalmodel | \datafamily{-\strategicuser})$.
  In particular, we make sure that its quadratic approximation near the optimum is associated to a definite positive matrix, whose eigenvalues are very different, and whose eigenvectors are slightly rotated from the canonical basis.
  This proves that, for any multiplicative gain, there are instances where a strategic user can obtain this multiplicative gain, in terms of drawing the global model (or other users' models) closer to their preferred model through data misreporting.
  Appendix~\ref{app:negative-theorem} provides a full construction of this worst case analysis, which is highly nontrivial.
  Note also that the asymptotic strategyproofness analysis will provide deeper insights into the phenomenon at play.
\end{proof}

\subsection{Proof sketch of the positive result}

\begin{proof}[Sketch of proof]
  Our assumptions allow to reduce strategyproofness to the one-dimension case.
  But then, in dimension 1, by behaving strategically, user $\strategicuser$ can only achieve values for $\globalmodel^\licch{}$ within a (possibly unbounded) interval $I$.
  But now, if $\attackmodel < \inf(I)$, then $\globalmodel^\licch{} = \inf(I)$.
  If $\attackmodel > \sup(I)$, then  $\globalmodel^\licch{} = \sup(I)$.
  Finally, if $\attackmodel \in I$, then $\globalmodel^\licch{} = \attackmodel$.
  In any case, the learned value $\globalmodel^\licch{}$ is closest to $\honestmodelsub{\strategicuser}$ when $\attackmodel = \honestmodelsub{\strategicuser}$.
  Similar arguments apply to biasing a target user $\targetuser$'s model $\modelsub{\targetuser}^\licch{}$.
  Appendix~\ref{app:strategyproofness} details the proof.
\end{proof}

\subsection{Approximate strategyproofness in the general case}
\label{sec:approximate_strategyproofness}

In general, unfortunately, local loss functions are not coordinate-wise separable.
Nevertheless, here, we discuss how our strategyproofness theorem can be leveraged to tune \licchavi{} and make it approximately strategyproof.
The main trick is to tune each user $\user$'s coordinate system depending on the sum of other users' reduced loss $\reducedloss{} (\globalmodel | \datafamily{-\strategicuser})$.

More precisely, denote $\hessian{\strategicuser} \triangleq \nabla^2_{|\globalmodel = \globalmodel^{-\strategicuser}} \reducedloss{} (\globalmodel | \datafamily{-\strategicuser})$,
where $\globalmodel^{-\strategicuser}$ is the output of \licchavi{} executed on all users apart from user $\strategicuser$.
Since $\reducedloss{}$ is convex, we know that $\hessian{\strategicuser}$ is semi-definite positive.
Moreover, it is symmetric, thus there exists an orthogonal matrix $\orthogonalmatrix_\strategicuser$
and eigenvalues $\eigenvalue{1}^\strategicuser \geq \ldots \geq \eigenvalue{\dimension}^\strategicuser \geq 0$ such that
$\hessian{\strategicuser} = \orthogonalmatrix_\strategicuser^T \diag(\eigenvalue{1}^\strategicuser, \ldots, \eigenvalue{\dimension}^\strategicuser) \orthogonalmatrix_\strategicuser$.
Then, assuming there are many users, so that the effect of strategic user $\strategicuser$ on the global model is small, and ignoring the additive constants,
the reduced \licchavi{} loss becomes approximately
\begin{align}
    \licch{}_{\strategicuser} (\globalmodel | \attackmodel, \datafamily{-\strategicuser})
    &\approx \licchaviweight \norm{ \attackmodel - \globalmodel }{1}
    + (\globalmodel - \globalmodel^{-\strategicuser})^T \hessian{\strategicuser} (\globalmodel - \globalmodel^{-\strategicuser}) \nonumber \\
    &\approx \licchaviweight \norm{ \attackmodel - \globalmodel }{1}
    + \sum_{\coordinate \in [\dimension]} \eigenvalue{\coordinate} ( (\orthogonalmatrix_\strategicuser \globalmodel)_\coordinate - (\orthogonalmatrix_\strategicuser \globalmodel^{-\strategicuser})_\coordinate )^2.
\end{align}
Now, in general, this loss has no guarantee of strategyproofness.
However, we may now tune \licchavi{} for strategic user $\strategicuser$ based on the orthogonal matrix $\orthogonalmatrix_\strategicuser$ to fall back on the previous case.
To do so, we introduce the following $\orthogonalmatrixfamily{}$-skewed \licchavi{} loss:
\begin{align}
    &\licch{} (\globalmodel, \modelfamily{} | \datafamily{}, \orthogonalmatrixfamily{})
    \triangleq \sum_{\user \in [\USER]} \localloss{\user} (\modelsub{\user} | \data{\user})
    + \licchaviweight\sum_{\user \in [\USER]} \huber_{\frac{\huberconstant}{\huberfunction{\card{\data{\user}}}}} \left(\orthogonalmatrix_\user \modelsub{\user} - \orthogonalmatrix_\user \globalmodel \right).
\end{align}
Indeed, this loss corresponds to the following reduced loss for model attack:
\begin{align*}
    \licch{}_{\strategicuser} (\globalmodel | \attackmodel, \datafamily{-\strategicuser}, \orthogonalmatrixfamily{}) 
    &= \licchaviweight \norm{ \orthogonalmatrix_\strategicuser \attackmodel - \orthogonalmatrix_\strategicuser \globalmodel }{1}
    + \reducedloss{} (\globalmodel | \datafamily{-\strategicuser}, \orthogonalmatrixfamily_{-\strategicuser} )  \nonumber \\
    &\approx \sum_{\coordinate \in [\dimension]} \licchaviweight \absv{ (\orthogonalmatrix_\strategicuser (\attackmodel - \globalmodel))_\coordinate }
    + \eigenvalue{\coordinate} ( (\orthogonalmatrix_\strategicuser \globalmodel)_\coordinate - (\orthogonalmatrix_\strategicuser \globalmodel^{-\strategicuser})_\coordinate )^2,
\end{align*}
assuming $\reducedloss{} (\globalmodel | \datafamily{-\strategicuser}, \orthogonalmatrixfamily_{-\strategicuser} ) \approx \reducedloss{} (\globalmodel | \datafamily{-\strategicuser} )$.
Importantly, this last approximation is coordinate-wise separable, which means that Theorem~\ref{th:strategyproof} would approximately apply here.

Unfortunately, the precise analysis of our approximations is highly nontrivial, and beyond the scope of this paper.
In particular, we leave open the problem of determining how to (efficiently) compute matrices $\orthogonalmatrixfamily{}$
such that the vectors $\orthogonalmatrix_\user^T \canonicalvector{\coordinate}$ are (approximately) eigenvectors of $\nabla^2_{| \globalmodel = \globalmodel^{-\user}} \reducedloss{} (\globalmodel | \datafamily{-\user}, \orthogonalmatrixfamily_{-\user} )$ for all users $\user \in [\USER]$,
where $\canonicalvector{\coordinate}$ is the $\coordinate$-th vector of the canonical basis $\canonicalbasis$.

%% file: asymptotic.tex
\section{Asymptotic Strategyproofness}
\label{sec:asymptotic}

In this section, we discuss the strategyproofness of \licchavi{} in the asymptotic setting of a large number of users.
From a practical standpoint, this is arguably the most relevant setting for
it allows us to approximate the loss restricted to other users by a quadratic function, as discussed below.

\subsection{Asymptotic setting}

Let us first define the asymptotic setting, which is inspired from~\cite{geometric_median}.
Intuitively, it corresponds to the limit where $\USER \rightarrow \infty$,
when each user's dataset $\data{\user}$ is drawn independently from a distribution of datasets $\datadistribution$.
This then naturally leads us to the following definition of strategyproofness which, for simplicity, we state in the case of model attack.
By our equivalence lemma (Lemma~\ref{lemma:equivalence_data_model}), it is evidently equivalent to its (more wordy) data attack version.

\begin{definition}
A GPL algorithm \alg{} is asymptotically global-targeted $\strategyproofbound$-strategyproof under distribution $\datadistribution$ if,
for any $\varepsilon, \delta > 0$ and any preferred model $\honestmodel_{\strategicuser}$,
there exists $\USER_0$ such that,
if there are $\USER -1 \geq \USER_0$ users (other than strategic user $\strategicuser$) whose datasets $\datafamily{-\strategicuser}$ are all drawn independently from $\datadistribution$,
then with probability at least $1-\delta$,
we have
\begin{align}
    \forall \attackmodel{} \mathsep &\norm{ \globalmodel^{\alg{}} (\honestmodel_\strategicuser, \datafamily{-\strategicuser}) - \honestmodel_\strategicuser }{2}
    \leq (1+\strategyproofbound) \norm{ \globalmodel^{\alg{}} (\attackmodel{}, \datafamily{-\strategicuser}) - \honestmodel_\strategicuser }{2} + \varepsilon.
\end{align}
\end{definition}

Now, when the number $\USER$ of users is large, the \licchavi{} loss under model attack can be approximated by
\begin{equation}
\label{eq:approximate}
    \licch{} (\globalmodel | \attackmodel{}, \datafamily{-\strategicuser})
    \approx \licchaviweight \norm{ \globalmodel - \attackmodel{} }{1}
    + (\USER - 1) \expectedreducedloss{} (\globalmodel),
\end{equation}
where $\expectedreducedloss{} (\globalmodel) \triangleq \expectVariable{\data{} \leftarrow \datadistribution}{ \reducedloss{} (\globalmodel | \data{}) }$,
with an expectation taken over the random dataset $\data{}$.

Now denote $\defaultglobalmodel
\triangleq
\argmin_{\globalmodel}
\expectedreducedloss{} (\globalmodel)$ the model obtained by ignoring the strategic user.
We also define the achievable set $\achievableset$ as the set of global models that could be obtained through model attack by the strategic user, i.e.
\begin{equation}
    \achievableset (\datafamily{-\strategicuser}) \triangleq \set{ \globalmodel^{\licch{}} (\attackmodel{}, \datafamily{-\strategicuser})  \st \attackmodel{} \in \setR^d }.
\end{equation}
When $\USER$ is large, the strategic user's attack model $\attackmodel{}$ will only have a small effect on the optimal global model.
This means that, for large values of $\USER$, $\achievableset(\datafamily{-\strategicuser})$ gets arbitrarily small.
As a result, over $\achievableset (\datafamily{-\strategicuser})$, and for a large enough number of users, the expected reduced loss $\expectedreducedloss{} (\globalmodel)$ in (\ref{eq:approximate}) can be approximated by a quadratic loss.
More precisely, defining $\defaultglobalmodel$ the minimum of $\expectedreducedloss{}$ and $\hessian{\infty} \triangleq \nabla^2 \expectedreducedloss{} (\defaultglobalmodel)$, we then have
\begin{align}
    &\licch_{\strategicuser} (\globalmodel | \attackmodel{}, \datafamily{-\strategicuser})
    \approx \licchaviweight \norm{ \globalmodel - \attackmodel{} }{1} 
    + (\USER - 1) (\globalmodel - \defaultglobalmodel)^T \hessian{\infty} (\globalmodel - \defaultglobalmodel).
\end{align}
Unfortunately, the precise formulation and derivation of this approximation is highly nontrivial, and left for future work.
Importantly, however, it suggests that we can restrict our attention to this quadratic setting.

\subsection{The quadratic setting}

In light of our discussion above, and without loss of generality in the asymptotic setting,
we now focus on \licchavi{} against a quadratic function, with a unit voting right, i.e.
\begin{equation}
    \licch{} (\globalmodel | \attackmodel{}, \sdpmatrix{}) \triangleq
    \norm{ \globalmodel - \attackmodel{} }{1}
    + \globalmodel^T \sdpmatrix \globalmodel.
\end{equation}

To state our result, we define the crookedness of $\sdpmatrix \succ 0$ by
\begin{equation}
\label{eq:crookedness}
    \crookedness(\sdpmatrix) \triangleq
    \sup_{\varx \in \setR^d}
    \inf_{\underset{\sign(\vary) = \sign(x)}{\vary \in \setR^d}}
    \frac{\norm{\varx}{2} \norm{\sdpmatrix \vary}{2}}{\varx^T \sdpmatrix \vary} -1,
\end{equation}
where $\sign$ applies the sign function on each coordinate (and thus implies $\vary_\coordinate = 0$ whenever $\varx_\coordinate = 0$).
We now have the following theorem.
\begin{theorem}
\label{th:licchavi_quadratic}
\licchavi{} against positive definite matrix $\sdpmatrix$ is $\crookedness (\sdpmatrix)$-strategyproof.
\end{theorem}

\begin{proof}[Sketch of proof]
The proof is nontrivial, as it involves understanding the function $\globalmodel^{\licch{}} (\attackmodel{}, \sdpmatrix)$,
as well as its image for $\attackmodel{} \in \setR^\dimension$, which is the \emph{achievable set}.
Arguments based on orthogonal projection then allow us to lower bound the distance between $\targetmodel{}$ and the achievable set.
The full proof is given in Appendix~\ref{app:asymptotic}.
\end{proof}

Unfortunately, \crookedness{} does not seem to yield a closed form formula.
Nevertheless, we point out that it takes lower values than another measure called \skewness{}, introduced by \cite{geometric_median}.

\begin{proposition}
\label{prop:crookedness_vs_skewness}
Let
    $\skewness (\sdpmatrix) \triangleq
    \sup_{\varx \in \setR^\dimension}
      \frac{\norm{\varx}{2} \norm{\sdpmatrix \varx}{2}}{ \varx^T \sdpmatrix \varx }  - 1$.
Then, for any $\sdpmatrix \succ 0$, we have
    $\crookedness(\sdpmatrix) \leq \skewness(\sdpmatrix)$.
Moreover, there are definite positive matrices $\sdpmatrix \succ 0$ for which the inequality is strict.
\end{proposition}

\begin{proof}[Sketch of proof]
The inequality is obtained by considering $\vary \triangleq \varx$ in Equation~\eqref{eq:crookedness}.
The strict inequality can be shown by considering a matrix $\sdpmatrix$ whose eigenvectors are the canonical basis vectors, and whose eigenvalues differ.
\end{proof}

Since \cite{geometric_median} essentially showed that the geometric median is $\skewness (\sdpmatrix)$-strategyproof,
and that this strategyproofness bound is tight,
our theorem nicely shows that the coordinate-wise median (and variants like \licchavi{}) is essentially \emph{more strategyproof} than the geometric median and its variants.
Intuitively, by forcing agreements to be coordinate-wise, the coordinate-wise median (and variants like \licchavi{}) restricts the vulnerabilities to what happens only along the canonical basis vectors.
In fact, in the specific case where each vector of the canonical basis $\canonicalbasis$ is an eigenvector of $\sdpmatrix$, but with different eigenvalues, then Theorem~\ref{th:strategyproof} actually applies, and \licchavi{} is strategyproof ($\strategyproofbound = 0$).
This is strictly better than what the geometric median guarantees in such a case.

%% file: experiment.tex
\section{Experimenting \licchavi{}}
\label{sec:experiment}

To test \licchavi{}, we consider a language fine-tuning task,
on a language model with an embedding layer of dimension 256, two GRU with hidden size 200 and a fully connected layer with 10'000 output units (vocab size) using \textit{softmax}, with cross-entropy loss on next token prediction.
This yields $5 \times 10^6$ free parameters, half being in the embedding layer.
A global model was pretrained on a pretraining dataset, and the model's embedding layer was frozen.

We then considered a real Twitter dataset made of
$2 \times 10^{7}$ hydrated tweets during the 2016 USA presidential election from $\USER = 100$ users.
We performed federated fine-tuning of the last layer, with users' tweets, using \licchavi{} (with $\licchaviweight \triangleq 1$ and $\huberconstant \triangleq 10$) and the $\ell_2^2$ baseline~\cite{DinhTN20,HanzelyHHR20},
which we implemented on top of Pytorch.
We used a batch size of 32, 3 epochs per nodes per round, and a learning rate of $10^{-3}$.
The performance was measured on another set of tweets by the $R_3$ measure, which is the average number of times our model contains the correct next word in its top 3 predictions.
The results are displayed in Figure~\ref{fig:huberi_L2^2}.

\begin{figure}[h]
    \vspace{-1mm}
    \centering
    \includegraphics[width=.7\columnwidth]{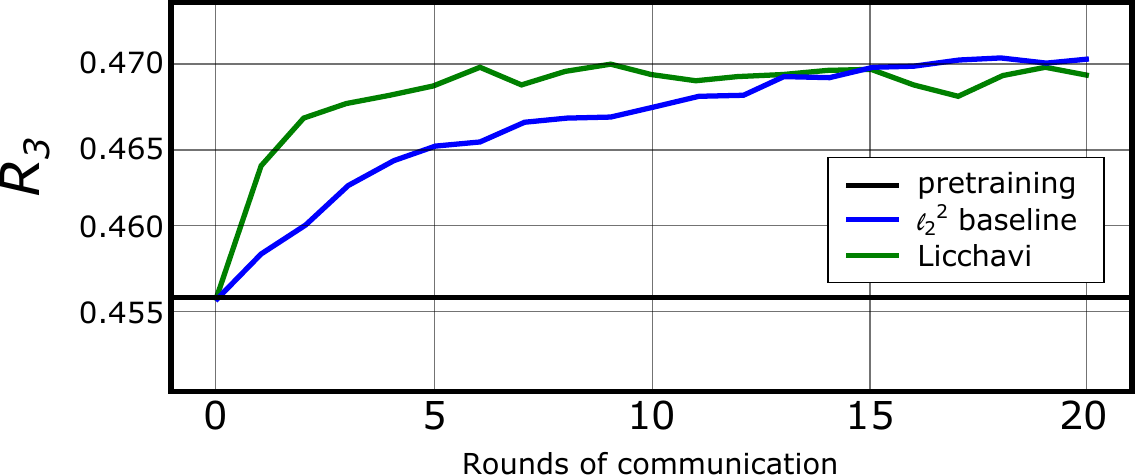}
    \caption{$R_3$ measures during fine-tuning.}
    \label{fig:huberi_L2^2}
    \vspace{-1mm}
\end{figure}

We observe that both \licchavi{} and $\ell_2^2$ fine tuning improve the $R_3$ measure of the global model in a similar way.
This suggests that \licchavi{} can provide similar performances as classical personalized federated learning models, while additionally providing strategyproofness guarantees.

%% file: conclusion.tex
\section{Conclusion}
\label{sec:conclusion}

We introduced \licchavi{}, an algorithm for global and personalized learning, and we analyzed its strategyproofness.
We proved both positive and negative theorems.
Perhaps most importantly, we showed that \licchavi{} yields some asymptotic $\strategyproofbound$-strategyproofness guarantees,
and we sketched how to guarantee approximate strategyproofness in the general setting, by tuning \licchavi{}.
We argue that such considerations are critical to guarantee the trustworthiness of training databases and, eventually, the security of deployed machine learning algorithms.
We also implemented \licchavi{} for language fine tuning, 
and our experiments highlighted its practicality and performance.

%% file: acknowledgment.tex
\section*{Acknowledgment}

The authors are thankful to Guillaume Le Mailloux for some useful preliminary work on strategyproof learning.

%% file: existence_uniqueness_proof.tex
\section{Existence and uniqueness of the optimum}
\label{sec:existence_uniqueness_proof}

\begin{repproposition}{prop:existence_uniqueness}
 For any family of datasets $\datafamily{}$, $\licch{}$ yields a unique minimum.
\end{repproposition}

\begin{proof}
Let us first prove the existence of the minimum. Define $L_0 \triangleq \licch{}(0,0|\datafamily{})$, the value of the the  \licchavi{} loss at $0$. Since the local loss functions are strongly convex, there exists a constant $c$ such that if for any $\user \in [\USER]$ we have $\norm{\modelsub{n}}{2}\geq c$, then $\localloss{} (\modelsub{\user} | \data{\user}) \geq L_0$. This implies that at the infimum, we must have $\norm{\modelsub{n}}{2}\leq c$, for all $\user \in [\USER]$. On the other hand, if $\norm{\globalmodel}{2} \rightarrow{} \infty $ then $\norm{\modelsub{n}}{2}\leq c$ implies that $\huber_{\frac{\huberconstant}{\huberfunction{\card{\data{\user}}}}} \left(\modelsub{\user} - \globalmodel \right)$ goes to infinity and in particular becumes larger than $L_0$ for $\norm{\globalmodel}{2}$ large enough.  Therefore, the infinum of \licch{} must be reached in a bounded and close region around the origin which is a compact set. The infimum is thus a minimum, which proves the existence of a minimum. 

We now move on to proving the uniqueness. 
Consider two minima $(\globalmodel^{(1)},\modelfamily{}^{(1)})$ and $(\globalmodel^{(2)},\modelfamily{}^{(2)})$.
By the strict convexity of $\localloss{}$ we have 
\begin{equation}
\label{eq:local_inequality}
    \forall \user \in [\USER] \mathsep  \localloss{} \left( \frac{\modelsub{\user}^{(1)}+\modelsub{\user}^{(2)} }{2} \st \data{\user} \right) 
    \leq \frac{1}{2}\left(\localloss{} (\modelsub{\user}^{(1)} | \data{\user})
    +\localloss{} (\modelsub{\user}^{(2)} | \data{\user}) \right),
\end{equation}
with strict inequality if $\modelsub{\user}^{(1)} \neq \modelsub{\user}^{(2)}$. Similarly, by the strict convexity of \huber{}, for all $\user \in [\USER]$, we obtain
\begin{equation}
\label{eq:huber_inequality}
    \huber_{\frac{\huberconstant}{\huberfunction{\card{\data{\user}}}}} \left(\frac{(\modelsub{\user}^{(1)} - \globalmodel^{(1)}) + (\modelsub{\user}^{(2)} - \globalmodel^{(2)})}{2}\right)  \leq \frac{1}{2}\left(\huber_{\frac{\huberconstant}{\huberfunction{\card{\data{\user}}}}} \left(\modelsub{\user}^{(1)} - \globalmodel^{(1)} \right)+\huber_{\frac{\huberconstant}{\huberfunction{\card{\data{\user}}}}} \left(\modelsub{\user}^{(2)} - \globalmodel^{(2)} \right) \right),
\end{equation}
with a strict inequality $\modelsub{\user}^{(1)} - \globalmodel^{(1)} \neq \modelsub{\user}^{(2)} - \globalmodel^{(2)}$. 
Now combining all of the above inequalities yields
\begin{equation}
    \licch{} \left(\frac{\globalmodel^{(1)}+\globalmodel^{(2)}}{2}, \frac{\modelfamily{}^{(1)}+\modelfamily{}^{(2)}}{2} \st \datafamily{}\right)
    \leq \frac{1}{2}\left(
            \licch{} \left(\globalmodel^{(1)}, \modelfamily{}^{(1)} \st \datafamily{}\right)
            +\licch{} \left(\globalmodel^{(2)}, \modelfamily{}^{(2)} | \datafamily{} \right)
    \right),
\end{equation}
and the above inequality becomes strict if at least one of the inequalities in (\ref{eq:local_inequality}) or (\ref{eq:huber_inequality}) are strict. 
But since, by optimality of the solutions, the right-hand side takes the minimum value of \licch{}, we must have equality.
This implies that $\modelsub{\user}^{(1)} = \modelsub{\user}^{(2)}$
and $\modelsub{\user}^{(1)} - \globalmodel^{(1)} = \modelsub{\user}^{(2)} - \globalmodel^{(2)}$  for all users $\user \in [\USER]$.
Considering any user, say $\user = 1$, in the second equality then implies $\globalmodel^{(1)} = \globalmodel^{(2)}$.
All in all, we thus have uniqueness.
\end{proof}

%% file: strategyproofness_proof.tex
\section{Reductions to model attacks}
\label{app:strategyproof_proof}

\subsection{Reduced losses}

\begin{replemma}{lemma:local_optimum}
For any data $\data{}$ and $\globalmodel$, the infimum problem defining $\reducedloss{}$ yields a unique minimum $\optmodelsub{} (\globalmodel, \data{})$.
\end{replemma}

\begin{proof}
  Given that the local loss $\localloss{}$ is strongly convex and that the pseudo-Huber loss is convex, we know that their sum is strongly convex, which guarantees the existence and uniqueness of $\optmodelsub{} (\globalmodel, \data{})$.
\end{proof}

\begin{replemma}{lemma:gradient_reduced_loss}
For any data $\data{}$, the reduced loss $\reducedloss{} (\globalmodel | \data{})$ is convex and differentiable.
Moreover, $\nabla \reducedloss{} = \licchaviweight \nabla \huber_{\frac{\huberconstant}{\huberfunction{\card{\data{}}}}} \left( \globalmodel - \optmodelsub{} (\globalmodel, \data{}) \right)$, and thus $\norm{\nabla \reducedloss{}}{\infty} \leq \licchaviweight$.
\end{replemma}

\begin{proof}
  The convexity and differentiability of the reduced loss follows straightforwardly from Lemma~9 of \cite{equivalence_data_gradient}.
  By the same lemma, we have $\nabla \reducedloss{} = \licchaviweight \nabla \huber_{\frac{\huberconstant}{\huberfunction{\card{\data{}}}}} \left( \globalmodel - \optmodelsub{} (\globalmodel, \data{}) \right)$.
  In particular, $\partial_{\globalmodel_\coordinate} \reducedloss{} = \licchaviweight{} \frac{\globalmodel_\coordinate - \optmodelsub{\coordinate}}{\sqrt{ \frac{\huberconstant^2}{1+\card{\data{}}} + \left( \globalmodel_\coordinate - \optmodelsub{\coordinate} \right)^2 }}$, 
  whose absolute value is at most $\licchaviweight{}$.
\end{proof}

\begin{replemma}{lemma:local_global_optimum}
$\globalmodel^{\licch{}} (\datafamily{})$ is the unique minimum of $\reducedloss{} (\globalmodel | \datafamily{})$, 
while $\modelsub{\user}^{\licch{}} (\datafamily{}) = \optmodelsub{} (\globalmodel^{\licch{}} (\datafamily{}), \data{\user})$.
\end{replemma}

\begin{proof}
  Clearly, we have
  \begin{equation}
      \inf_{\globalmodel, \modelfamily} \licch{} (\globalmodel, \modelfamily | \datafamily{})
      = \inf_{\globalmodel} \set{ \inf_{\modelfamily} \licch{} (\globalmodel, \modelfamily | \datafamily{}) }
      = \inf_{\globalmodel} \reducedloss{} (\globalmodel | \datafamily{}).
  \end{equation}
  This shows that $\globalmodel$ minimizes \licch{} (with some value of $\modelfamily$) if and only if it minimizes the reduced loss $\reducedloss{}$.
  Since the former has a unique minimum, so does the latter, which is $\globalmodel^{\licch{}} (\datafamily{})$.
  Moreover, similar computations clearly show that 
  \begin{equation}
        \licch{} \left( \globalmodel^{\licch{}} (\datafamily{}), \modelfamily^{\licch{}} (\datafamily{}) \st \datafamily{} \right) 
        = \licch{} \left( \globalmodel^{\licch{}} (\datafamily{}), \optmodelfamily (\globalmodel^{\licch{}} (\datafamily{}), \data{\user}) \st \datafamily{} \right).
  \end{equation}
  By the uniqueness of the minimum, we then conclude that $\modelfamily^{\licch{}} (\datafamily{}) = \optmodelfamily (\globalmodel^{\licch{}} (\datafamily{}), \data{\user})$.
  Or, put differently, for each user $\user$, we have $\modelsub{\user}^{\licch{}} (\datafamily{}) = \optmodelsub{\user} (\globalmodel^{\licch{}} (\datafamily{}), \data{\user})$.
\end{proof}

\subsection{Strong local PAC*}

In this section, to prove Lemma~\ref{lemma:strong_local_pac}, we prove an even stronger result, 
which asserts that, assuming user $\user$ provides enough data, then, given any global model, $\optmodelsub{\user}$ is successfully probably approximately correct.
This result will be useful in the proof of Lemma~\ref{lemma:equivalence_data_model}.

\begin{lemma}
\label{lemma:general_strong_local_pac}
  Assume gradient PAC* local losses. 
  Then, for any model $\honestmodelsub{\user}$ and any $\varepsilon, \delta > 0$, there exists $\USERINPUT{}$ such that,
  if user $\user$ provides a dataset $\data{\user}^\dagger$ with at least $\USERINPUT{}$ answers to random queries with model $\honestmodelsub{\user}$, 
  then, with probability at least $1-\delta$, we have
  \begin{equation}
      \forall \globalmodel \in \setR^\dimension \mathsep 
      \norm{ \optmodelsub{\user} \left(\globalmodel, \data{\user}^\dagger\right) - \honestmodelsub{\user} }{2} \leq \varepsilon.
  \end{equation}
\end{lemma}

\begin{proof}
  Consider a user $\user \in [\USER]$ and their preferred model $\honestmodelsub{\user}$. Fix $\varepsilon, \delta >0$. 
  Define $K \triangleq \norm{\honestmodelsub{\user}}{2}$.
  Denote $\USERINPUT{}$ the number of data points provided by  user $\user$.
  By the optimality of $\optmodelsub{\user} \triangleq \optmodelsub{\user} \left(\globalmodel, \data{\user}^\dagger\right)$, we have
\begin{align}
    0&\in 
    (\optmodelsub{\user} - \honestmodelsub{\user})^T \nabla  \localloss{} (\optmodelsub{\user} | \data{\user}^\dagger) 
    + (\optmodelsub{\user} - \honestmodelsub{\user})^T \nabla_{\modelsub{\user}} \left( \licchaviweight \huber_{\frac{\huberconstant}{\huberfunction{\card{\data{\user}^\dagger}}}} \left(\optmodelsub{\user} - \optglobalmodel \right) \right) \\
    &\geq (\optmodelsub{\user} - \honestmodelsub{\user})^T \nabla  \localloss{} (\optmodelsub{\user} | \data{\user}^\dagger) 
    - \norm{\optmodelsub{\user} - \honestmodelsub{\user}}{2} \norm{\nabla_{\modelsub{\user}} \left( \licchaviweight \huber_{\frac{\huberconstant}{\huberfunction{\card{\data{\user}^\dagger}}}} \left(\optmodelsub{\user} - \optglobalmodel \right) \right)}{2} \\
    &\geq (\optmodelsub{\user} - \honestmodelsub{\user})^T \nabla  \localloss{} (\optmodelsub{\user} | \data{\user}^\dagger)  - \licchaviweight \sqrt{d} \norm{\optmodelsub{\user} - \honestmodelsub{\user}}{2},
    \label{eq:condition_pac_strong}
\end{align}
where, in the last line, we used the fact the infinite norm of the gradient $\huber$ is bounded by $1$ and $\norm{\model}{2}\leq \sqrt{d} \norm{\model}{\infty}$.

Now, gradient PAC* implies the existence of an event $\mathcal E$ that occurs with probability at least $\Probability(K, \USERINPUT{})$, under which we have

\begin{equation}
    0\geq A_K \USERINPUT{} \min \left\lbrace \norm{\optmodelsub{\user} - \honestmodelsub{\user}}{2}, \norm{\optmodelsub{\user} - \honestmodelsub{\user}}{2}^2 \right\rbrace 
    - B_K \USERINPUT{}^\alpha \norm{\optmodelsub{\user} - \honestmodelsub{\user}}{2} - \licchaviweight \sqrt{d} \norm{\optmodelsub{\user} - \honestmodelsub{\user}}{2}.
\end{equation}
Note that the event $\event$ is independent from $\globalmodel$.
If $\norm{\optmodelsub{\user} - \honestmodelsub{\user}}{2}  \geq 1$, this implies
\begin{equation}
    0 \geq (A_K \USERINPUT{}  - B_K \USERINPUT{}^\alpha - \licchaviweight \sqrt{d})
    \norm{\optmodelsub{\user} - \honestmodelsub{\user}}{2},
\end{equation}
which cannot hold for $\USERINPUT{} > \USERINPUT{1}\triangleq\max \set{2\licchaviweight \sqrt{d}/A_K,(2B_K/A_K)^\frac{1}{1-\alpha}}$. Thus, for $\USERINPUT{}> \USERINPUT{1}$, we have

\begin{equation}
    0\geq A_K \USERINPUT{}  \norm{\optmodelsub{\user} - \honestmodelsub{\user}}{2}^2 
    - (B_K \USERINPUT{}^\alpha + \licchaviweight \sqrt{d} ) \norm{\optmodelsub{\user} - \honestmodelsub{\user}}{2}.
\end{equation}
As a result, 
\begin{equation}
  \norm{\optmodelsub{\user} - \honestmodelsub{\user}}{2} \leq \frac{B_K \USERINPUT{}^\alpha + \licchaviweight \sqrt{d}}{A_{K} \USERINPUT{}}.
\end{equation}
Considering $\USERINPUT{}$ large enough such that $\USERINPUT{} > \USERINPUT{1}$ and $\Probability(K, \USERINPUT{}) \geq 1-\delta$ and $\frac{B_K \USERINPUT{}^\alpha + \licchaviweight \sqrt{d}}{A_{K} \USERINPUT{}}\leq \varepsilon$, we obtain the result.
\end{proof}

Lemma~\ref{lemma:strong_local_pac} then follows straightforwardly.

\begin{replemma}{lemma:strong_local_pac}
  For gradient PAC* local losses, \licchavi{} is strongly local PAC*.
\end{replemma}

\begin{proof}
  This follows from Lemma~\ref{lemma:general_strong_local_pac}, 
  and the fact that $\modelsub{\user}^\licch (\datafamily{}) = \optmodelsub{\user} (\globalmodel^\licch (\datafamily{}), \data{\user})$ (Lemma~\ref{lemma:local_global_optimum}).
\end{proof}

\subsection{Equivalence between data attack and model attack}
\label{app:equivalence}

Our equivalence proof will leverage the following lemma, largely drawn from~\cite{equivalence_data_gradient}.

\begin{lemma}[Lemma 1 from \cite{equivalence_data_gradient}]
\label{lemma:reduction_model_to_data}
  Consider any data $\datafamily{}$ and any user $\strategicuser \in [\USER]$.
  Then having user $\strategicuser$ reporting $\data{\strategicuser}$ is equivalent to having them reporting the model $\modelsub{\strategicuser}^\licch{} (\datafamily{})$, i.e.
  \begin{equation}
    \globalmodel^\licch{} (\data{\strategicuser}, \datafamily{-\strategicuser}) = \globalmodel^\licch{} (\modelsub{\strategicuser}^\licch{} (\datafamily{}), \datafamily{-\strategicuser} )
    \qquad \text{and} \qquad 
    \forall \user \mathsep \modelsub{\user}^\licch{} (\data{\strategicuser}, \datafamily{-\strategicuser}) = \modelsub{\user}^\licch{} (\modelsub{\strategicuser}^\licch{} (\datafamily{}), \datafamily{-\strategicuser} ).
  \end{equation}
\end{lemma}

\begin{proof}[Sketch of proof]
  This is derived from the fact that the loss as a function of $\globalmodel$ and $\modelfamily_{-\strategicuser}$ is unchanged.
\end{proof}

We will also need the following lemma, adapted from Lemma 2 of \cite{equivalence_data_gradient} (or, rather, by from its generalization, which is Lemma 14 in \cite{equivalence_data_gradient}).
However, a bit more work is needed to adapt their proof, as, here, we need to transform a pseudo-Huber loss into an $\ell_1$ loss.
We bound this transformation by the following uniform bound.

\begin{lemma}
\label{lemma:pseudo_huber_uniform_convergence}
For any $\delta > 0$ and $t \in \setR$, we have $0 \leq \sqrt{\delta^2 +t^2} - \absv{t} \leq \delta$.
\end{lemma}

\begin{proof}
  Clearly, $\delta^2 + t^2 \geq t^2$, which implies $\sqrt{\delta^2 + t^2} \geq \absv{t}$, and thus $\sqrt{\delta^2 +t^2} - \absv{t} \geq 0$.
  Moreover, we have $(\sqrt{\delta^2 +t^2} - \sqrt{t^2})^2 = \delta^2 + t^2 - 2 \sqrt{t^2 \delta^2 + t^4} + t^2 \leq \delta^2 + 2 t^2 - 2\sqrt{t^4} = \delta^2$,
  using the inequality $\sqrt{t^2 \delta^2 + t^4} \geq \sqrt{t^4}$.
  Taking the square root yields the lemma.
\end{proof}

\begin{lemma}
We have $0 \leq \huber_\delta (x) - \norm{x}{1} \leq \delta \dimension$.
\end{lemma}

\begin{proof}
  By the previous lemma, on each coordinate $\coordinate$, we have $0 \leq \huber_\delta (x_\coordinate) - \absv{x_\coordinate}{1} \leq \delta \dimension$.
  Adding up all the coordinates yields the lemma.
\end{proof}

\begin{lemma}
\label{lemma:uniform_strong_pac}
  Assume strong local PAC* learning.
  Consider a user $\strategicuser \in [\USER]$, any model $\targetmodel$ and fix other users' datasets $\datafamily{-\strategicuser}$.
  For any $\varepsilon, \delta > 0$,
  there exists $\USERINPUT{}$ such that,
  if user $\strategicuser$ provides a dataset $\data{\strategicuser}^\dagger$ by answering at least $\USERINPUT{}$ random queries with model $\targetmodel$,
  then with probability at least $1-\delta$, we have
  \begin{equation}
    \norm{ \globalmodel^\licch{} (\targetmodel, \datafamily{-\strategicuser}) - \globalmodel^\licch{} (\data{\strategicuser}^\dagger, \datafamily{-\strategicuser}) }{2} \leq \varepsilon
    \quad \text{and} \quad 
    \forall \user \mathsep 
    \norm{ \modelsub{\user}^\licch{} (\targetmodel, \datafamily{-\strategicuser}) - \modelsub{\user}^\licch{} (\data{\strategicuser}^\dagger, \datafamily{-\strategicuser}) }{2} \leq \varepsilon.
  \end{equation}
\end{lemma}

\begin{proof}
  Define the compact set $C$ of models that are $\varepsilon$-close to $\globalmodel^\licch{} \triangleq \globalmodel^\licch (\targetmodel, \datafamily{-\strategicuser})$ and $\modelfamily_{-\strategicuser}^\licch{} \triangleq \modelfamily_{-\strategicuser}^\licch (\targetmodel, \datafamily{-\strategicuser})$, i.e.
  \begin{equation}
      C \triangleq \set{ (\globalmodel, \modelfamily) \st 
        \norm{\globalmodel - \globalmodel^\licch (\targetmodel, \datafamily{-\strategicuser})}{2} \leq \varepsilon 
        ~\text{and}~ 
        \forall \user \neq \strategicuser \mathsep  \norm{\modelsub{\user} - \modelsub{\user}^\licch (\targetmodel, \datafamily{-\strategicuser})}{2} \leq \varepsilon
      }.
  \end{equation}
  Denote $D \triangleq \overline{\setR^{\dimension \times (1+\USER)} - C}$ the closure of the complement of $C$.
  By the same arguments as Proposition~\ref{prop:existence_uniqueness}, we know that $\licch{} (\globalmodel, \modelfamily_{-\strategicuser} | \targetmodel, \datafamily{-\strategicuser})$ yields a minimum over $D$.
  But by the uniqueness of the minimum, we know that 
  \begin{equation}
      \eta \triangleq 
      \inf_{(\globalmodel, \modelfamily_{-\strategicuser}) \in D} \licch{} (\globalmodel, \modelfamily_{-\strategicuser} | \targetmodel, \datafamily{-\strategicuser})
      - \licch{} (\globalmodel^\licch{}, \modelfamily_{-\strategicuser}^\licch{} | \targetmodel, \datafamily{-\strategicuser}) > 0.
  \end{equation}
  Thus, for any $(\globalmodel, \modelfamily) \in D$, we have
  $\licch{} (\globalmodel, \modelfamily | \targetmodel, \datafamily{-\strategicuser}) \geq \eta + \licch{} (\globalmodel^\licch{}, \modelfamily^\licch{} | \targetmodel, \datafamily{-\strategicuser})$.
  We now invoke strong local PAC* learning.
  More precisely, consider the event
  \begin{equation}
      \event \triangleq \set{
        \forall \globalmodel \in \setR^\dimension \mathsep
        \norm{ \optmodelsub{\strategicuser} (\globalmodel, \data{\strategicuser}^\dagger) - \targetmodel }{2} \leq \min \set{\varepsilon, \eta / 6 \licchaviweight}
      }.
  \end{equation}
  By Lemma~\ref{lemma:general_strong_local_pac}, we know that there exists $\USERINPUT{1}$ such that, 
  if user $\strategicuser$ provides a dataset $\data{\strategicuser}^\dagger$ at least $\USERINPUT{1}$ answers to random queries, 
  then the event $\event$ occurs with probability at least $1-\delta$.
  Now consider $\USERINPUT{} \triangleq \max \set{\USERINPUT{1}, 9 \licchaviweight^2 \dimension^2 \huberconstant^2 / \eta^2}$.
  We now assume that the dataset $\data{\strategicuser}^\dagger$ contains at least $\USERINPUT{}$ answers to random queries.
  Then $\event$ still occurs with probability at least $1-\delta$.
  By optimality of $\modelsub{\strategicuser}^\licch{} (\globalmodel, \data{\strategicuser}^\dagger)$,  under $\event$, we then have 
  \begin{equation}
      \forall \modelsub{\strategicuser}, \globalmodel \in \setR^\dimension \mathsep 
      \localloss{} (\optmodelsub{\strategicuser} (\globalmodel, \data{\strategicuser}^\dagger) | \data{\strategicuser}^\dagger) 
      + \licchaviweight \huber_{\frac{\huberconstant}{\sqrt{1+ \card{\data{\strategicuser}^\dagger}}} } (\globalmodel - \optmodelsub{\strategicuser} (\globalmodel, \data{\strategicuser}^\dagger))
      \leq 
      \localloss{} (\modelsub{\strategicuser} | \data{\strategicuser}^\dagger) 
      + \licchaviweight \huber_{\frac{\huberconstant}{\sqrt{1+ \card{\data{\strategicuser}^\dagger}}} } (\globalmodel - \modelsub{\strategicuser})
  \end{equation}
  Given Lemma~\ref{lemma:local_global_optimum}, applying this inequality to $\globalmodel \triangleq \globalmodel^\licch{}$ and $\modelsub{\strategicuser} \triangleq \optmodelsub{\strategicuser}(\globalmodel, \data{\strategicuser}^\dagger)$ then yields
  \begin{equation}
      \localloss{} (\optmodelsub{\strategicuser}(\globalmodel, \data{\strategicuser}^\dagger) | \data{\strategicuser}^\dagger) 
      \geq 
      \localloss{} (\modelsub{\strategicuser}^\licch{} | \data{\strategicuser}^\dagger) 
      + \licchaviweight \huber_{\frac{\huberconstant}{\sqrt{1+ \card{\data{\strategicuser}^\dagger}}} } (\globalmodel^\licch - \modelsub{\strategicuser}^\licch{})
      - \licchaviweight \huber_{\frac{\huberconstant}{\sqrt{1+ \card{\data{\strategicuser}^\dagger}}} } (\globalmodel^\licch{} - \optmodelsub{\strategicuser}(\globalmodel, \data{\strategicuser}^\dagger))
  \end{equation}
  Then, for any models $(\globalmodel, \modelfamily_{-\strategicuser}) \in D$ and $\modelsub{\strategicuser}$, 
  under $\event$, we then have
  \begin{align}
      &\licch{} ( \globalmodel, \modelfamily | \data{\strategicuser}^\dagger, \datafamily{-\strategicuser} )
      \geq \licch{} ( \globalmodel, \modelfamily_{-\strategicuser}, \optmodelsub{\strategicuser}(\globalmodel, \data{\strategicuser}^\dagger) | \data{\strategicuser}^\dagger, \datafamily{-\strategicuser} ) \\
      &= \left( \sum_{\user \neq \strategicuser} \localloss{} (\modelsub{\user} | \data{\user}) + \licchaviweight \huber_{\frac{\huberconstant}{\sqrt{1+ \card{\data{\user}}}} } (\globalmodel - \modelsub{\user}) \right)
      + \localloss{} (\optmodelsub{\strategicuser}(\globalmodel, \data{\strategicuser}^\dagger) | \data{\strategicuser}^\dagger) 
      + \licchaviweight \huber_{\frac{\huberconstant}{\sqrt{1+ \card{\data{\strategicuser}^\dagger}}} } \left(\globalmodel - \optmodelsub{\strategicuser}(\globalmodel, \data{\strategicuser}^\dagger) \right) \\
      &\geq \left( \licch{} (\globalmodel, \modelfamily_{-\strategicuser} | \targetmodel, \datafamily{-\strategicuser})
      - \licchaviweight \norm{ \globalmodel - \targetmodel }{1} \right) \nonumber \\
            &\qquad \qquad + \left( \localloss{} (\modelsub{\strategicuser}^\licch{} | \data{\strategicuser}^\dagger) 
                            + \licchaviweight \huber_{\frac{\huberconstant}{\sqrt{1+ \card{\data{\strategicuser}^\dagger}}} } (\globalmodel^\licch{} - \modelsub{\strategicuser}^\licch{})
                            - \licchaviweight \huber_{\frac{\huberconstant}{\sqrt{1+ \card{\data{\strategicuser}^\dagger}}} } (\globalmodel^\licch{} - \optmodelsub{\strategicuser}(\globalmodel, \data{\strategicuser}^\dagger)) \right) \nonumber  \\
            &\qquad \qquad + \licchaviweight \huber_{\frac{\huberconstant}{\sqrt{1+ \card{\data{\strategicuser}^\dagger}}} } \left(\globalmodel - \optmodelsub{\strategicuser}(\globalmodel, \data{\strategicuser}^\dagger) \right) \\
      &\geq \eta
      + \licch{} (\globalmodel^\licch{}, \modelfamily_{-\strategicuser}^\licch{} | \targetmodel, \datafamily{-\strategicuser})
      + \localloss{} ( \modelsub{\strategicuser}^\licch{} | \data{\strategicuser}^\dagger)
      + \licchaviweight \huber_{\frac{\huberconstant}{\sqrt{1+ \card{\data{\strategicuser}^\dagger}}} } (\globalmodel^\licch{} - \modelsub{\strategicuser}^\licch{}) 
      - \licchaviweight \norm{ \globalmodel - \targetmodel }{1} \nonumber \\
            &\qquad \qquad + \licchaviweight 
                \huber_{\frac{\huberconstant}{\sqrt{1+ \card{\data{\strategicuser}^\dagger}}} } (\globalmodel - \optmodelsub{\strategicuser}(\globalmodel, \data{\strategicuser}^\dagger)) 
            - \licchaviweight \huber_{\frac{\huberconstant}{\sqrt{1+ \card{\data{\strategicuser}^\dagger}}} } \left(\globalmodel^\licch{} - \optmodelsub{\strategicuser}(\globalmodel, \data{\strategicuser}^\dagger) \right) \\
      &\geq \eta +
      \licch{} (\globalmodel^\licch{}, \modelfamily^\licch{} | \data{\strategicuser}^\dagger, \data{-\strategicuser})
      + \licchaviweight \norm{ \globalmodel^\licch{} - \targetmodel }{1}
      - \licchaviweight \norm{ \globalmodel - \targetmodel }{1} \nonumber \\
            &\qquad \qquad + \licchaviweight \norm{\globalmodel - \optmodelsub{\strategicuser}(\globalmodel, \data{\strategicuser}^\dagger)}{1} 
                - \left( \licchaviweight \norm{\globalmodel^\licch{} - \optmodelsub{\strategicuser}(\globalmodel, \data{\strategicuser}^\dagger)}{1} + \frac{\licchaviweight \huberconstant \dimension}{\sqrt{1+ \card{\data{\strategicuser}^\dagger}}} \right) \\
      &\geq \eta +
      \licch{} (\globalmodel^\licch{}, \modelfamily^\licch{} | \data{\strategicuser}^\dagger, \data{-\strategicuser})
      - \frac{\licchaviweight \dimension \huberconstant}{\sqrt{1 + \USERINPUT{} }} \nonumber \\
            &\qquad \qquad - \licchaviweight \absv{ \norm{ \globalmodel^\licch{} - \targetmodel }{1}  - \norm{\globalmodel^\licch{} - \optmodelsub{\strategicuser}(\globalmodel, \data{\strategicuser}^\dagger)}{1} } 
                            - \licchaviweight \absv{ \norm{\globalmodel - \optmodelsub{\strategicuser}(\globalmodel, \data{\strategicuser}^\dagger)}{1} - \norm{\globalmodel - \targetmodel}{1} } \\
      &\geq \licch{} (\globalmodel^\licch{}, \modelfamily^\licch{} | \data{\strategicuser}^\dagger, \data{-\strategicuser})
      + \eta - \frac{\licchaviweight \dimension \huberconstant}{\sqrt{1 + \frac{9 \licchaviweight^2 \dimension^2 \huberconstant^2}{\eta^2} }} - 2 \licchaviweight \norm{\targetmodel - \optmodelsub{\strategicuser}(\globalmodel, \data{\strategicuser}^\dagger)}{1} \\
      &\geq \licch{} (\globalmodel^\licch{}, \modelfamily^\licch{} | \data{\strategicuser}^\dagger, \data{-\strategicuser})
      + \eta - \frac{\eta}{3} - \frac{2\eta}{6}
      = \licch{} (\globalmodel^\licch{}, \modelfamily^\licch{} | \data{\strategicuser}^\dagger, \data{-\strategicuser})
      + \frac{\eta}{3} \\
      &> \licch{} (\globalmodel^\licch{}, \modelfamily^\licch{} | \data{\strategicuser}^\dagger, \data{-\strategicuser}).
  \end{align}
This proves that any $(\globalmodel, \modelfamily_{-\strategicuser}) \in D$ cannot be the unique minimum of \licchavi{} given datasets $(\data{\strategicuser}^\dagger, \data{-\strategicuser})$.
Thus $(\globalmodel^\licch{} (\data{\strategicuser}^\dagger, \data{-\strategicuser}), \modelfamily_{-\strategicuser}^\licch{} (\data{\strategicuser}^\dagger, \data{-\strategicuser})) \in C$.
Adding to this the guarantee of event $\event$ yields the lemma.
\end{proof}

\begin{replemma}{lemma:equivalence_data_model}
Assuming strong local PAC*, \licchavi{} is global-targeted $\strategyproofbound$-strategyproof under data attack if and only if it is global-targeted $\strategyproofbound$-strategyproof under model attack.
The equivalence also holds for user-targeted $\strategyproofbound$-strategyproofness.
\end{replemma}

\begin{proof}
  Let us first assume that \licchavi{} is global-targeted $\strategyproofbound$-strategyproof under model attack.
  We then fix $\varepsilon, \delta > 0$,
  and we consider the event $\event$ defined by
  \begin{equation}
    \event \triangleq \set{ 
        \forall \datafamily{-\strategicuser} \mathsep
        \norm{ 
            \globalmodel^\licch{} (\data{\strategicuser}^\dagger, \datafamily{-\strategicuser}) 
            - \globalmodel^\licch{} (\modelsub{\strategicuser}^\dagger, \datafamily{-\strategicuser})
        }{2} \leq \frac{\varepsilon}{1+\strategyproofbound}
    }.
  \end{equation}
  Note that $\event$ is random because it depends on the random honest dataset $\data{\strategicuser}^\dagger$, whose random queries are answered with model $\targetmodel$.
  Given strong local PAC*, we know that there is $\USERINPUT{}$ large enough such that $\probability{\event} \geq 1 - \delta$.
  Assume $\event$.
  Now fix other users' datasets $\datafamily{-\strategicuser}$,
  and consider any strategic dataset $\data{\strategicuser}^\spadesuit$ that $\strategicuser$ could inject.
  By Lemma~\ref{lemma:reduction_model_to_data}, we know that there exists $\attackmodel \triangleq \modelsub{\strategicuser}^\licch (\data{\strategicuser}^\spadesuit, \datafamily{-\strategicuser})$ such that $\globalmodel^\licch{} (\attackmodel, \datafamily{-\strategicuser}) = \globalmodel^\licch{} (\data{\strategicuser}^\spadesuit, \datafamily{-\strategicuser})$.
  Then
  \begin{align}
      \norm{ \globalmodel^\licch{} (\data{\strategicuser}^\spadesuit, \datafamily{-\strategicuser}) - \targetmodel }{2} 
      &\leq \norm{ \globalmodel^\licch{} (\attackmodel, \datafamily{-\strategicuser}) - \targetmodel }{2} \\ 
      &\leq (1+\strategyproofbound) \norm{ \globalmodel^\licch{} (\targetmodel, \datafamily{-\strategicuser}) - \targetmodel }{2} \\ 
      &\leq (1+\strategyproofbound) \left(
        \norm{ \globalmodel^\licch{} (\targetmodel, \datafamily{-\strategicuser}) - \globalmodel^\licch{} (\data{\strategicuser}^\dagger, \datafamily{-\strategicuser}) }{2}
        + \norm{ \globalmodel^\licch{} (\data{\strategicuser}^\dagger, \datafamily{-\strategicuser}) - \targetmodel }{2}
      \right) \\ 
      &\leq \varepsilon + (1+\strategyproofbound) \norm{ \globalmodel^\licch{} (\data{\strategicuser}^\dagger, \datafamily{-\strategicuser}) - \targetmodel }{2},
  \end{align}
  which proves $\strategyproofbound$-strategyproofness under data attack.

  Reciprocally, assume that \licchavi{} is global-targeted $\strategyproofbound$-strategyproof under data attack.
  Fix any target model $\targetmodel$, attack model $\attackmodel$ and any $\varepsilon > 0$.
  We then define the following events, which depend on the datasets $\data{\strategicuser}^\dagger$ and $\data{\strategicuser}^\spadesuit$, 
  whose random queries are answered respectively with models $\targetmodel$ and $\attackmodel$:
  \begin{align}
      \event_1 &\triangleq \set{ 
            \forall \data{\strategicuser}, \data{-\strategicuser} \mathsep 
            \norm{ \globalmodel^\licch{} (\data{\strategicuser}^\dagger, \datafamily{-\strategicuser}) - \targetmodel }{2}
            \leq (1+\strategyproofbound) \norm{ \globalmodel^\licch{} (\data{\strategicuser}, \datafamily{-\strategicuser}) - \targetmodel }{2} + \varepsilon 
        }, \\
      \event_2 &\triangleq \set{ 
            \forall \data{-\strategicuser} \mathsep 
            \norm{ \globalmodel^\licch{} (\data{\strategicuser}^\dagger, \datafamily{-\strategicuser}) - \globalmodel^\licch{} (\targetmodel, \datafamily{-\strategicuser}) }{2} \leq \varepsilon
        }, \\
      \event_3 &\triangleq \set{ 
            \forall \data{-\strategicuser} \mathsep 
            \norm{ \globalmodel^\licch{} (\data{\strategicuser}^\spadesuit, \datafamily{-\strategicuser}) - \globalmodel^\licch{} (\attackmodel, \datafamily{-\strategicuser}) }{2} \leq \varepsilon
        }.
  \end{align}
  By $\strategyproofbound$-strategyproofness under data attack, we know that, when the datasets answer sufficiently many queries, $\event_1$ occurs with probability at least $3/4$.
  By Lemma~\ref{lemma:uniform_strong_pac}, we also know that, when the datasets answer sufficiently many queries, each of events $\event_2$ and $\event_3$ also occurs with probability at least $3/4$.
  As a result, we know that, when the datasets answer sufficiently many queries, the intersection $\event_1 \cap \event_2 \cap \event_3$ occurs with probability at least $1/4$.
  Under $\event_1 \cap \event_2 \cap \event_3$, we then have
  \begin{align}
    \forall \data{-\strategicuser} \mathsep 
    &\norm{ \globalmodel^\licch{} (\targetmodel, \datafamily{-\strategicuser}) - \targetmodel }{2}
    \leq \norm{ \globalmodel^\licch{} (\targetmodel, \datafamily{-\strategicuser}) - \globalmodel^\licch{} (\data{\strategicuser}^\dagger, \datafamily{-\strategicuser}) }{2} 
        + \norm{ \globalmodel^\licch{} (\data{\strategicuser}^\dagger, \datafamily{-\strategicuser}) - \targetmodel }{2} \\
    &\leq \varepsilon + (1+\strategyproofbound) \norm{ \globalmodel^\licch{} (\data{\strategicuser}^\spadesuit, \datafamily{-\strategicuser}) - \targetmodel }{2} \\
    &\leq \varepsilon + (1+\strategyproofbound) \left( 
        \norm{ \globalmodel^\licch{} (\data{\strategicuser}^\spadesuit, \datafamily{-\strategicuser}) - \globalmodel^\licch{} (\attackmodel, \datafamily{-\strategicuser}) }{2} 
        + \norm{ \globalmodel^\licch{} (\attackmodel, \datafamily{-\strategicuser}) - \targetmodel }{2} 
    \right) \\
    &\leq \varepsilon + (1+\strategyproofbound) \varepsilon + (1+\strategyproofbound) \norm{ \globalmodel^\licch{} (\attackmodel, \datafamily{-\strategicuser}) - \targetmodel }{2} .
  \end{align}
But this event is deterministic. Since it occurs with a positive probability, it must thus hold with probability 1.
We conclude by noting that it holds for any $\varepsilon > 0$.
Taking the limit $\varepsilon \rightarrow 0$ proves global-targeted $\strategyproofbound$-strategyproofness under model attack.
  
  The proof for user-targeted $\strategyproofbound$-strategyproofness is essentially the same.
\end{proof}

\input{non-strategyproofness}

\section{Proof of strategyproofness}
\label{app:strategyproofness}

In this section, we now prove Theorem~\ref{th:strategyproof}, namely, the strategyproofness of \licchavi{} for gradient PAC* and coordinate-wise separable local losses.

\subsection{Disentangling the coordinates}

In this section, we show how the assumption of coordinate-wise separable local loss functions allows to reduce the study of strategyproofness to one-dimensional functions.
Namely, recall that the local losses are coordinate-wise separable if $\localloss{} (\model | \data{}) = \sum_{i \in [\dimension]} \localloss{i} (\model_i | \data{})$.
We can then define the coordinate-wise \licchavi{} loss function along dimension $\coordinate$ by
\begin{align}
  \licch_{\strategicuser \coordinate} (\globalmodel_\coordinate | \modelsub{\strategicuser \coordinate}^\spadesuit, \datafamily{-\strategicuser})
  \triangleq \sum_{\user \neq \strategicuser} \localloss{\coordinate} (\modelsub{\user \coordinate} | \data{\user})
  + \licchaviweight \absv{\globalmodel_\coordinate - \modelsub{\strategicuser \coordinate}^\spadesuit}.
\end{align}
The global loss function is then the sum of the coordinate-wise loss functions, i.e.,
\begin{equation}
  \licch_{\strategicuser} (\globalmodel | \attackmodel, \datafamily{-\strategicuser}) 
  = \sum_{\coordinate \in [\dimension]} \licch_{\strategicuser \coordinate} (\globalmodel_\coordinate | \modelsub{\strategicuser \coordinate}^\spadesuit, \datafamily{-\strategicuser}).
\end{equation}
From this, we derive trivially the following lemma.

\begin{lemma}
  $\globalmodel_\coordinate^{\licch} (\attackmodel, \datafamily{-\strategicuser})$ minimizes $\licch_{\strategicuser \coordinate} (\globalmodel_\coordinate | \modelsub{\strategicuser \coordinate}^\spadesuit, \datafamily{-\strategicuser})$.
\end{lemma}

\begin{proof}
  This is straightforward.
\end{proof}

\subsubsection{Strategyproofness in dimension 1}

In particular, this means that the strategic user $\strategicuser$ can focus on coordinate-wise attacks.

\begin{lemma}
  \label{lemma:dimension_1_minimizer}
    Consider a strictly convex function $f : \setR \rightarrow \setR_+$, and denote $\optglobalmodel(\attackmodel) = \argmin_{\globalmodel \in \setR} \set{f(\globalmodel) + \absv{\globalmodel - \attackmodel}}$. 
    Then there exists $\optglobalmodel_{min}, \optglobalmodel_{max} \in \setR \cup \set{-\infty,+\infty}$, with $\optglobalmodel_{min} \leq \optglobalmodel_{max}$, such that, 
    $\optglobalmodel((-\infty,\optglobalmodel_{min}]) = \set{\optglobalmodel_{min}}$, 
    $\optglobalmodel((\optglobalmodel_{max}, +\infty]) = \set{\optglobalmodel_{max}}$
    and $\optglobalmodel(\attackmodel) = \attackmodel$ for all $\attackmodel \in [\optglobalmodel_{min}, \optglobalmodel_{max}]$.
\end{lemma}

\begin{proof}
Denote $F(\globalmodel, \attackmodel) \triangleq f(\globalmodel) + \absv{\globalmodel - \attackmodel}$.
First, let us verify that $\optglobalmodel(\cdot)$ is well-defined, by showing that, for all $\attackmodel$, $F(\cdot, \attackmodel)$ has a unique minimum.
For $\absv{\globalmodel - \attackmodel} \geq f(\attackmodel)$, we then have $F(\globalmodel,\attackmodel) \geq f(\globalmodel) + \absv{\globalmodel - \attackmodel} \geq f(\attackmodel)$.
Thus, the infinum of $F$ on $\setR$ is its infinum on $[\attackmodel - f(\attackmodel), \attackmodel + f(\attackmodel)]$, which is a compact set.
Thus the infinum is reached by a minimum $\optglobalmodel(\attackmodel) \in [\attackmodel - f(\attackmodel), \attackmodel + f(\attackmodel)]$.
The uniqueness of $\optglobalmodel(\attackmodel)$ is then guaranteed by the strict convexity of $f$, which implies that of $F$.

Let us now show that $\optglobalmodel$ must be nondecreasing.
Since $f$ is strictly convex, its subgradients $\partial f$ are nondecreasing.
The same holds for $\sign(\cdot) = \partial \absv{\cdot}$.
Now assume $\modelsub{\strategicuser 1}^\spadesuit \leq \modelsub{\strategicuser 2}^\spadesuit$.
Then, $0 \in \partial_1 F(\optglobalmodel (\modelsub{\strategicuser 1}^\spadesuit), \modelsub{\strategicuser 1}^\spadesuit) = \partial f(\optglobalmodel (\modelsub{\strategicuser 1}^\spadesuit)) + \sign(\optglobalmodel (\modelsub{\strategicuser 1}^\spadesuit) - \modelsub{\strategicuser 1}^\spadesuit) 
\geq \partial f(\optglobalmodel (\modelsub{\strategicuser 1}^\spadesuit)) + \sign(\optglobalmodel (\modelsub{\strategicuser 1}^\spadesuit) - \modelsub{\strategicuser 2}^\spadesuit) = \partial_1 F(\optglobalmodel (\modelsub{\strategicuser 1}^\spadesuit), \modelsub{\strategicuser 2}^\spadesuit)$.
Thus the subderivatives $F(\cdot, \modelsub{\strategicuser 2}^\spadesuit)$ at $\globalmodel = \optglobalmodel (\modelsub{\strategicuser 1}^\spadesuit)$ are negative or nil.
This implies that the optimum of $F(\cdot, \modelsub{\strategicuser 2}^\spadesuit)$ is on the right of $\optglobalmodel(\modelsub{\strategicuser 1}^\spadesuit)$.
In other words, we must have $\optglobalmodel(\modelsub{\strategicuser 1}^\spadesuit) \leq \optglobalmodel(\modelsub{\strategicuser 2}^\spadesuit)$.

Let us now define $\optglobalmodel_{min} \triangleq \inf_{\attackmodel \in \setR^d} \optglobalmodel(\attackmodel) \in \setR \cup \set{-\infty}$ 
and $\optglobalmodel_{max} \triangleq \sup_{\attackmodel \in \setR^d} \optglobalmodel(\attackmodel) \in \setR \cup \set{+\infty}$.
Now consider $\attackmodel \in (\optglobalmodel_{min}, \optglobalmodel_{max})$.
We thus know that there exists $\modelsub{\strategicuser 1}^\spadesuit, \modelsub{\strategicuser 2}^\spadesuit \in \setR$ such that $\optglobalmodel(\modelsub{\strategicuser 1}^\spadesuit) \leq \attackmodel \leq \optglobalmodel(\modelsub{\strategicuser 2}^\spadesuit)$.
By the monotonicity of $\optglobalmodel$, we know that $\modelsub{\strategicuser 1}^\spadesuit \leq \attackmodel \leq \modelsub{\strategicuser 2}^\spadesuit$.
Moreover, the optimality of $\optglobalmodel(\modelsub{\strategicuser 1}^\spadesuit)$ implies that $0 \in \partial f(\optglobalmodel (\modelsub{\strategicuser 1}^\spadesuit)) + \sign(\optglobalmodel (\modelsub{\strategicuser 1}^\spadesuit) - \modelsub{\strategicuser 1}^\spadesuit) \geq \partial f(\optglobalmodel (\modelsub{\strategicuser 1}^\spadesuit)) - 1$,
since the minimal value of the sign function is $-1$.
Similarly, by the optimality of $\optglobalmodel(\modelsub{\strategicuser 2}^\spadesuit)$, we have $\partial f(\optglobalmodel(\modelsub{\strategicuser 2}^\spadesuit)) \leq 1$.
Since $\partial f$ is nondecreasing, we must then have $\partial f(\attackmodel) \subset [-1,1]$. 
But then, denoting $g \in \partial f(\attackmodel)$, since $\sign(\attackmodel-\attackmodel) = \sign(0) = [-1,1]$, we have $\partial_1 F(\attackmodel, \attackmodel) \supset g + [-1,1] = [g -1, g +1]$.
Since $g \in [-1,1]$, we know that $\partial_1 F(\attackmodel, \attackmodel)$ intersects 0, which proves that $\optglobalmodel(\attackmodel) = \attackmodel$.

Now consider $\attackmodel < \optglobalmodel_{min}$.
By the definition of $\optglobalmodel_{min}$, we know that $\optglobalmodel_{min} \leq \optglobalmodel(\attackmodel)$.
As a result, $\sign(\optglobalmodel(\attackmodel) - \attackmodel) = -1$
We then know that $0 \in \partial_1 F(\optglobalmodel(\attackmodel), \attackmodel) = \partial f(\optglobalmodel(\attackmodel)) -1$.
But note that this equality property holds for all $\attackmodel < \optglobalmodel_{min}$.
Therefore, $\optglobalmodel(\optglobalmodel_{min} - 1) = \optglobalmodel(\attackmodel)$ for all $\attackmodel < \optglobalmodel_{min}$.
But since $\optglobalmodel$ is nondecreasing, we also know that $\optglobalmodel_{min} = \lim_{\attackmodel \rightarrow \infty} \optglobalmodel(\attackmodel) = \optglobalmodel(\optglobalmodel_{min} - 1)$.
Thus, in fact, for any $\attackmodel < \optglobalmodel_{min}$, we have $\optglobalmodel(\attackmodel) = \optglobalmodel_{min}$.
From this, it also follows that $\partial f(\optglobalmodel_{min})$ contains $-1$, which implies that $\optglobalmodel(\optglobalmodel_{min}) = \optglobalmodel_{min}$.

Finally, we deal similarly with the case of $\optglobalmodel_{max}$.
Namely, similarly, we show that for all $\attackmodel \geq \optglobalmodel_{max}$, we have $\optglobalmodel(\attackmodel) = \optglobalmodel_{max}$.
\end{proof}

\begin{lemma}[Strategyproofness in dimension 1]
\label{lemma:strategyproof_median}
  Consider a strictly convex function $f : \setR \rightarrow \setR_+$, and denote $\optglobalmodel(\attackmodel) = \argmin_{\globalmodel \in \setR} \set{f(\globalmodel) + \absv{\globalmodel - \attackmodel}}$. 
  Then reporting $\attackmodel$ honestly minimizes the distance to the honest preferences, i.e., 
  \begin{equation}
    \forall \honestmodel, \attackmodel \in \setR \mathsep \absv{\optglobalmodel(\honestmodel{}) - \honestmodel} \leq \absv{\optglobalmodel(\attackmodel) - \honestmodel{}}.
  \end{equation}
\end{lemma}

\begin{proof}

  As in Lemma \ref{lemma:dimension_1_minimizer}, we simply distinguish the three cases $\honestmodel \leq \optglobalmodel_{min}$, $\optglobalmodel_{min} \leq \honestmodel \leq \optglobalmodel_{max}$ and $\honestmodel \leq \optglobalmodel_{max}$.
  In the second case, the left-hand side of the lemma is zero, which makes the inequality clear.
  In the first and third case, the left-hand side is equal to $\absv{\optglobalmodel_{min} - \honestmodel}$ and $\absv{\honestmodel - \optglobalmodel_{max}}$ respectively.
  The inequality then follows from the definition of $\optglobalmodel_{min}$ and $\optglobalmodel_{max}$.
\end{proof}

\begin{lemma}
\label{lemma:strategyproof_median_median}
  Consider two strictly convex functions $f$ and $g$ (we also allow $g = 0$). 
  We define $\optglobalmodel(\attackmodel) \triangleq \argmin_{\globalmodel} f(\globalmodel) + \absv{\globalmodel - \attackmodel}$ and $\optmodelsub{\targetuser}(\globalmodel) \triangleq \argmin_{\modelsub{\targetuser}} g(\modelsub{\targetuser}) + \absv{\globalmodel - \modelsub{\targetuser}}$.
  Then,
  \begin{equation}
  \label{eq:strategyproof_median_median}
    \forall \honestmodelsub{\strategicuser}, \attackmodel \in \setR \mathsep 
    \absv{\optmodelsub{\targetuser}(\optglobalmodel(\honestmodelsub{\strategicuser})) - \honestmodelsub{\strategicuser}}
    \leq \absv{\optmodelsub{\targetuser}(\optglobalmodel(\attackmodel)) - \honestmodelsub{\strategicuser}}.
  \end{equation}
\end{lemma}

\begin{proof}
  Denote $\optglobalmodel_{min} \triangleq \min \optglobalmodel$ and $\optglobalmodel_{max} \triangleq \max \optglobalmodel$ the minimal and maximal values of $\optglobalmodel$.
  By Lemma \ref{lemma:dimension_1_minimizer}, for $\honestmodelsub{\strategicuser} \in [\optglobalmodel_{min}, \optglobalmodel_{max}]$ (or if $g=0$), 
  we know that $\absv{\optmodelsub{\targetuser}(\globalmodel) - \honestmodelsub{\strategicuser}}$ is minimized for $\globalmodel = \honestmodelsub{\strategicuser}$, 
  which is achieved by reporting $\attackmodel = \honestmodelsub{\strategicuser}$.
  
  Now assume $\honestmodelsub{\strategicuser} \leq \optglobalmodel_{min}$.
  By Lemma \ref{lemma:dimension_1_minimizer}, for any $\attackmodel$, 
  we know that $\optglobalmodel (\attackmodel) \geq \optglobalmodel_{min} = \optglobalmodel(\honestmodelsub{\strategicuser})$.
  Then, by monotonicity of $\optmodelsub{\targetuser}$ (Lemma \ref{lemma:dimension_1_minimizer}), 
  then, for any $\attackmodel$, we have $\optmodelsub{\targetuser} (\optglobalmodel(\attackmodel)) \geq \optmodelsub{\targetuser} (\optglobalmodel(\honestmodelsub{\strategicuser}))$.

  Now, if $\optmodelsub{\targetuser} (\optglobalmodel(\honestmodelsub{\strategicuser})) = \max_{\globalmodel} \optmodelsub{\targetuser} (\globalmodel)$, 
  then we must have $\optmodelsub{\targetuser} (\optglobalmodel(\attackmodel)) = \optmodelsub{\targetuser} (\optglobalmodel(\honestmodelsub{\strategicuser}))$, and thus Equation (\ref{eq:strategyproof_median_median}) is actually an equality (and thus the inequality holds).
  Otherwise, if $\optmodelsub{\targetuser} (\optglobalmodel(\honestmodelsub{\strategicuser})) < \max_{\globalmodel} \optmodelsub{\targetuser} (\globalmodel)$, 
  then by Lemma \ref{lemma:dimension_1_minimizer}, we must have $\optmodelsub{\targetuser} (\optglobalmodel(\honestmodelsub{\strategicuser})) \geq \optglobalmodel(\honestmodelsub{\strategicuser})$.
  We then have $\optmodelsub{\targetuser} (\optglobalmodel(\attackmodel)) \geq \optmodelsub{\targetuser} (\optglobalmodel(\honestmodelsub{\strategicuser})) \geq \optglobalmodel(\honestmodelsub{\strategicuser}) \geq \optglobalmodel_{min} \geq \honestmodelsub{\strategicuser}$.
  In particular, we thus have $\optmodelsub{\targetuser} (\optglobalmodel(\attackmodel)) \geq \optmodelsub{\targetuser} (\optglobalmodel(\honestmodelsub{\strategicuser})) \geq \honestmodelsub{\strategicuser}$, from which the lemma follows.
  
  The case $\honestmodelsub{\strategicuser} \geq \optglobalmodel_{max}$ is derived similarly.
\end{proof}

\subsection{Combining it all}

\begin{lemma}
\label{lemma:invariant_norm}
  If $\absv{\varxsub{\coordinate}} \geq \absv{\varysub{\coordinate}}$ for all coordinates $\coordinate \in [D]$,
  then $\norm{\varx}{2} \geq \norm{\vary}{2}$.
\end{lemma}

\begin{proof}
  This is clear, given that $\norm{\varx}{2}^2 = \sum \absv{\varxsub{\coordinate}}^2$ is an increasing function of all terms $\absv{\varxsub{\coordinate}}$.
\end{proof}

\begin{reptheorem}{th:strategyproof}
    Assume that the local losses are gradient PAC* and coordinate-wise separable.
    Then \licchavi{} is both global and user-targeted strategyproof.
\end{reptheorem}

\begin{proof}[Proof of Theorem \ref{th:strategyproof}]
  We apply lemmas \ref{lemma:strategyproof_median} and \ref{lemma:strategyproof_median_median} with functions 
  \begin{align}
    f_{\coordinate} (\globalmodel_{\coordinate}) 
    &\triangleq \frac{1}{\regweightsub{\strategicuser} \diagonal{\strategicuser \coordinate}} 
    \min_{\modelfamily_{-\strategicuser \coordinate}} \licch_{-\strategicuser \coordinate} (\globalmodel_{\coordinate}, \modelfamily_{-\strategicuser \coordinate}, \datafamily{-\strategicuser \coordinate}), \\
    g_{\targetuser \coordinate} (\modelsub{\targetuser \coordinate})
    &\triangleq \frac{1}{\regweightsub{\targetuser} \diagonal{\targetuser \coordinate}} 
    \localloss{\targetuser \coordinate} (\modelsub{\targetuser \coordinate}).
  \end{align}
  From this it follows that, for any $\coordinate \in [d]$, any target user $\targetuser \in [\USER]$, any honest parameter $\honestmodelsub{\strategicuser}$ and any strategic vector $\attackmodel$, we have 
  \begin{align}
    \absv{\optglobalmodel_\coordinate(\honestmodelsub{\strategicuser}) - \honestmodelsub{\strategicuser \coordinate}} 
    &\leq \absv{\optglobalmodel_\coordinate(\attackmodel, \datafamily{-\strategicuser}) - \honestmodelsub{\strategicuser \coordinate}}, \\
    \absv{\optmodelsub{\targetuser}(\optglobalmodel_\coordinate(\honestmodelsub{\strategicuser})) - \honestmodelsub{\strategicuser \coordinate}} 
    &\leq \absv{\optmodelsub{\targetuser} (\optglobalmodel_\coordinate(\attackmodel, \datafamily{-\strategicuser})) - \honestmodelsub{\strategicuser \coordinate}}.
  \end{align}
  Combining Lemma~\ref{lemma:equivalence_data_model} and Lemma~\ref{lemma:invariant_norm} then allows to conclude. 
\end{proof}

%% file: non-strategyproofness.tex
\section{Proof of non-strategyproofness}
\label{app:negative-theorem}

To prove Theorem~\ref{th:negative_strategyproofness}, we propose a counter example, which will be parametrized by $A > 1$ (and we will consider the limit $A \rightarrow \infty$.)

\subsection{The counter example}

Namely, consider $\dimension = 2$, $\USER = 3$ and $\licchaviweight = 1$.
Now assume that users 1 and 2 are honest, and provide the same dataset $\data{} = \data{1} = \data{2}$ of at least $A^4/\huberconstant^2$ inputs, and for which
\begin{equation}
    \localloss{} (\model | \data{}) 
    = \frac{A^2}{2} (A \modelsub{1} - \modelsub{2})^2 + \frac{1}{2} \modelsub{2} + \frac{1}{2A^2} (\modelsub{1}^2 + \modelsub{2}^2).
\end{equation}
It is clear that this loss is strongly convex and differentiable, and thus satisfies the assumptions of the paper.
Moreover, intuitively, it locks $\model^*(\globalmodel)$ essentially along the line $A \modelsub{1} = \modelsub{2}$, 
while favoring lower values of $\modelsub{2}$ along this line, at least while $\modelsub{2} \geq - A^2$.

Moreover, since the loss looks the same from user 1 and user 2's perspectives, and by uniqueness of the minimum,
we know that, for any model attack $\attackmodel$ by strategic user $\strategicuser \triangleq 3$, 
we will have $\modelsub{1}^\licch (\attackmodel, \datafamily{-\strategicuser}) = \modelsub{2}^\licch (\attackmodel, \datafamily{-\strategicuser})$.
Thus, without loss of generality, we assume that both users are always assigned the same model $\model$.
In particular, denoting $\nu \triangleq \huberconstant / \sqrt{1 + \card{\data{}}} \leq 1/A^2$, 
and assuming strategic user $\strategicuser$ reports model $\model^\clubsuit$ (with $\clubsuit \in \set{\dagger, \spadesuit}$),
we can consider the following modified \licchavi{} loss (we leave the dependence on $\data{}$ implicit):
\begin{equation}
    \licch{} (\globalmodel, \model | \model^\clubsuit) 
    \triangleq 2 \localloss{} (\model | \data{}) + 2 \huber_\nu (\globalmodel - \model) + \norm{\globalmodel - \model^\clubsuit}{1}.
\end{equation}
Indeed, it is immediate to verify that the minimum $\globalmodel^\licch, \model^\licch$ of this loss 
will coincide with the \licchavi{} computation, 
i.e., $\globalmodel^\licch = \globalmodel^\licch (\model^\clubsuit, \datafamily{-\strategicuser})$
and $\model^\licch = \modelsub{\user}^\licch (\model^\clubsuit, \datafamily{-\strategicuser})$ for $\user \in \set{1,2}$.

We consider the target model $\model^\dagger \triangleq (0, 1)$, and the attack model $\model^\spadesuit \triangleq (1/A, 1)$.
We will show that the strategic user can get both $\globalmodel^\licch$ and $\model^\licch$ much closer to $\model^\dagger$, by reporting $\model^\spadesuit$ rather than $\model^\dagger$.
More precisely, we will prove that $\norm{\globalmodel^\licch (\model^\dagger) - \model^\dagger}{2} = \Omega(1)$ as $A \rightarrow \infty$,
while $\norm{\globalmodel^\licch (\model^\spadesuit) - \model^\dagger}{2} = \mathcal O(1/A)$.
This will prove Theorem~\ref{th:negative_strategyproofness}.

\subsection{Bounding the optimal global model}

In this section, we prove that $\globalmodel^\licch \approx \model^\licch$.
In fact, we will prove that for any fixed value of $\model$, if we optimize $\globalmodel$, then the distance between $\model$ and the optimized value $\globalmodel^*$ will be at most $1/A^2$.
Intuitively, this should not be surprising; indeed since the honest users 1 and 2 form a majority, they should be deciding where $\globalmodel^\licch$ is.
To prove this, denote $\unitvector{} \triangleq \nabla_\globalmodel \huber_\nu (\globalmodel - \model) \in (-1,1)^2$.
The partial derivatives with respect to the global model, given strategic user $\strategicuser$'s reported model $\model^\clubsuit$, are then given by
\begin{align}
    \partial_{\globalmodel_1} \licch{} &= 2 \unitvector{1} + \sign (\globalmodel_1 - \model^\clubsuit_1), 
    \label{eq:partial_global_licch1} \\
    \partial_{\globalmodel_1} \licch{} &= 2 \unitvector{2} + \sign (\globalmodel_2 - \model^\clubsuit_2).
    \label{eq:partial_global_licch2}
\end{align}

\begin{lemma}
For $\coordinate \in [\dimension] = \set{1,2}$, 
either $\globalmodel_\coordinate^\licch = \modelsub{\coordinate}^\clubsuit$ 
or $\globalmodel_\coordinate^\licch - \model_\coordinate^\licch = \sign (\globalmodel_\coordinate - \model^\clubsuit_\coordinate) \nu$.
\end{lemma}

\begin{proof}
If $\globalmodel_\coordinate^\licch \neq \modelsub{\coordinate}^\clubsuit$, then $\sign (\globalmodel_\coordinate - \model^\clubsuit_\coordinate) \in \set{-1,1}$.
By the optimality condition, we know that equations~\eqref{eq:partial_global_licch1} and~\eqref{eq:partial_global_licch2} must equal zero.
This implies that $\unitvector{\coordinate}^\licch{} = - \frac{1}{2} \sign (\globalmodel_\coordinate - \model^\clubsuit_\coordinate)$.
Solving this yields the lemma.
\end{proof}

Denote $\optglobalmodel(\model, \model^\clubsuit)$ the optimal value of $\globalmodel$ when $\model$ is fixed, and given the strategic user's reported model $\model^\clubsuit$.

\begin{lemma}
\label{lemma:bounded_optglobal_given_model}
$\norm{\optglobalmodel(\model, \model^\clubsuit) - \model}{\infty} \leq 1/A^2$.
\end{lemma}

\begin{proof}
By the optimality condition on $\optglobalmodel(\model)$, we know that for each coordinate $\coordinate$, 
we must have $0 \in \partial_{\globalmodel_\coordinate} \licch{} = 2 \unitvector{\coordinate} + \sign (\globalmodel_\coordinate - \modelsub{\coordinate}^\clubsuit)$.
Since $\sign (\globalmodel_\coordinate - \modelsub{\coordinate}^\clubsuit) \subset [-1,1]$, 
there must thus exist $\kappa_\coordinate \in [-1,1]$ such that $2 \unitvector{\coordinate} + \kappa_\coordinate = 0$,
which implies that $\unitvector{\coordinate} = - \frac{1}{2} \kappa_\coordinate \in [-1/2, 1/2]$.
Thus in particular $\absv{\unitvector{\coordinate}} = \frac{\absv{\globalmodel_\coordinate - \model_\coordinate}}{\sqrt{\nu^2 + (\globalmodel_\coordinate - \model_\coordinate)^2}} \leq 1/2$.
This implies that, at the optimum, $\left( 1 + \frac{\nu^2}{(\optglobalmodel_\coordinate (\model, \model^\clubsuit) - \model_\coordinate)^2} \right) \leq 1/2$,
which can only occur if $\absv{\optglobalmodel_\coordinate (\model, \model^\clubsuit) - \model_\coordinate} \leq \nu \leq 1/A^2$.
This is the lemma.
\end{proof}

\begin{lemma}
$\huber_\nu (\optglobalmodel(\model) - \model) \leq 2 \sqrt{2} / A^2$.
\end{lemma}

\begin{proof}
This follows straightforwardly from the previous lemma.
\end{proof}

\subsection{Model reduced loss}

The previous lemmas prompt us to consider the following model-reduced loss
\begin{equation}
    \modelreducedloss{} (\model | \model^\clubsuit) 
    \triangleq \inf_{\globalmodel} \licch{} (\globalmodel, \model | \model^\clubsuit)
    = \licch{} (\optglobalmodel(\model), \model | \model^\clubsuit).
\end{equation}
Note that we can write $\modelreducedloss{} (\model | \model^\clubsuit)
    = \modelreducedloss{simple} (\model | \model^\clubsuit)
    + \errorfunction{} (\model | \model^\clubsuit)$,
where $\modelreducedloss{simple} (\model | \model^\clubsuit) \triangleq 2 \localloss{} (\model | \data{}) + \norm{\model - \model^\clubsuit}{1}$ is what we will call the \emph{simplified model reduced loss},
and where $\errorfunction{}$ is the error function due to model reduced loss simplification, given by
\begin{equation}
    \errorfunction{} (\model | \model^\clubsuit) 
    \triangleq 2 \huber_\nu (\optglobalmodel(\model, \model^\clubsuit) - \model) 
    + \norm{\optglobalmodel(\model, \model^\clubsuit) - \model^\clubsuit}{1}
    - \norm{\model - \model^\clubsuit}{1}.
\end{equation}
Interestingly, the error function is uniformly small, so that we can essentially know $\modelreducedloss{}$ by only studying $\modelreducedloss{simple}$.

\begin{lemma}
For any $\model, \model^\clubsuit$, we have $\absv{\errorfunction{} (\model | \model^\clubsuit)} \leq 7/A^2$.
\end{lemma}

\begin{proof}
By triangle inequality, we have
\begin{align}
    \absv{\errorfunction{} (\model | \model^\clubsuit)}
    &\leq 2 \huber_\nu (\optglobalmodel(\model, \model^\clubsuit) - \model) 
    + \absv{\norm{\optglobalmodel(\model, \model^\clubsuit) - \model^\clubsuit}{1}
    - \norm{\model - \model^\clubsuit}{1}} \\
    &\leq 2 \frac{2 \sqrt{2}}{A^2} + \norm{\optglobalmodel(\model, \model^\clubsuit) - \model}{1}
    \leq \frac{4 \sqrt{2}}{A^2} + 2 \norm{\optglobalmodel(\model, \model^\clubsuit) - \model}{\infty} \\
    &\leq \frac{4 \sqrt{2} + 2}{A^2} \leq 7/A^2,
\end{align}
where we used the two previous lemmas.
\end{proof}

Given the lemma, we can provide the following bounds on interesting values of the reduced loss $\modelreducedloss{}$:
\begin{align}
    \modelreducedloss{} (0|\model^\dagger) 
    &= \modelreducedloss{simple} (0|\model^\dagger) + \errorfunction{} (0|\model^\dagger) \leq 1 + 7 A^{-2}, \\
    \modelreducedloss{} (\model^\spadesuit|\model^\spadesuit) 
    &= \modelreducedloss{simple} (\model^\spadesuit|\model^\spadesuit) + \errorfunction{} (\model^\spadesuit|\model^\spadesuit)
    \leq 1 + \frac{A^{-2} + 1}{A^2} + 7 A^{-2} \leq 1 + 9 A^{-2},
\end{align}
using $A > 1$.
In particular, if we can guarantee that $\modelreducedloss{} (\model|\model^\dagger) > 1 + 9 A^{-2}$ for $\model$ in some regions of space, then we can exclude the possibility that $\model^\licch (\model^\dagger)$ belongs there.

\subsection{The optimal model is bounded along the second coordinate}

\begin{lemma}
\label{lemma:bounded_coordinate2}
Consider any $\model^\clubsuit$ and suppose $A \geq \norm{\model^\clubsuit}{1} + 7$. 
Then $\absv{ \modelsub{2}^\licch{} (\model^\clubsuit) } \leq 2 A^2$.
\end{lemma}

\begin{proof}
First note that 
\begin{equation}
    \modelreducedloss{} ( 0 | \model^\clubsuit) 
    = \norm{\model^\clubsuit}{1} + \errorfunction{} (0 | \model^\clubsuit)
    \leq \norm{\model^\clubsuit}{1} + 7 A^{-2}
    \leq \norm{\model^\clubsuit}{1} + 7,
\end{equation}
using $A > 1$.
Now assume that $\absv{\modelsub{2}} \geq 2 A^2$, and consider any $\modelsub{1} \in \setR$.
Then
\begin{align}
    \modelreducedloss{} (\model | \model^\clubsuit)
    &\geq \modelsub{2} + \frac{1}{A^2} \modelsub{2}^2
    \geq \frac{1}{A^2} \absv{\modelsub{2}}^2 - \absv{\modelsub{2}}
    \geq 4 A^2 - 2 A^2 = 2 A^2 > A \geq \norm{\model^\clubsuit}{1} + 7,
\end{align}
using $A > 1$.
Thus $\modelreducedloss{} (\model | \model^\clubsuit) > \modelreducedloss{} ( 0 | \model^\clubsuit)$,
which implies that $\model$ cannot be optimal.
Thus we must have $\absv{ \modelsub{2}^\licch{} (\model^\clubsuit) } \leq 2 A^2$.
\end{proof}

\subsection{Further model reduced loss}

Now, interestingly, the simplified reduced loss $\modelreducedloss{simple}$ has a simple closed form, which allows us to study it directly.
In particular, given a a fixed value of $\modelsub{2}$, 
the parameter $\modelsub{1}$ is easily optimized with respect to $\modelreducedloss{simple}$.
Indeed, note that 
\begin{equation}
    \partial_{\modelsub{1}} \modelreducedloss{simple} 
    = 2 A^3 (A \modelsub{1} - \modelsub{2}) + \frac{2}{A^2} \modelsub{1} + \sign (\modelsub{1} - \modelsub{1}^\clubsuit).
\end{equation}
Thus, defining $\modelsub{1}^*(\modelsub{2}, \modelsub{1}^\clubsuit) \triangleq \argmin_{\modelsub{1}} \modelreducedloss{simple} (\model | \model^\clubsuit)$,
we must have $\left(2 A^4 + \frac{2}{A^2} \right) \modelsub{1}^*(\modelsub{2}, \modelsub{1}^\clubsuit)
    \in 2 A^3 \modelsub{2} - \sign (\modelsub{1} - \modelsub{1}^\clubsuit)$,
which then implies
\begin{equation}
    \modelsub{1}^*(\modelsub{2}, \modelsub{1}^\clubsuit) 
    \in \frac{\modelsub{2}}{A + A^{-5}} - \frac{\sign (\modelsub{1} - \modelsub{1}^\clubsuit)}{2 A^4 + 2 A^{-2}}
\end{equation}
Define $\errorfunction{2} (\modelsub{2} | \modelsub{1}^\clubsuit) \triangleq \modelsub{1}^*(\modelsub{2}, \modelsub{1}^\clubsuit) - A^{-1} \modelsub{2}$ the error when estimating $\modelsub{1}^*$ with $A^{-1} \modelsub{2}$, we then have the following bound on this error function.

\begin{lemma}
For all $\modelsub{2}, \modelsub{1}^\clubsuit$, 
we have $\absv{\errorfunction{2} (\modelsub{2} | \modelsub{1}^\clubsuit)} \leq A^{-6} \absv{\modelsub{2}} + A^{-4}$.
\end{lemma}

\begin{proof}
Indeed, we have
\begin{align}
    \absv{ \errorfunction{2} (\modelsub{2} | \modelsub{1}^\clubsuit) }
    &\leq \absv{ \frac{\modelsub{2}}{A} - \frac{\modelsub{2}}{A + A^{-5}} } + \frac{1}{2A^4 +2A^{-5}} 
    \leq \frac{A^{-5} \absv{\modelsub{2}}}{2 A} + \frac{1}{2A^4}
    \leq A^{-6} \absv{\modelsub{2}} + A^{-4},
\end{align}
which is the lemma.
\end{proof}

\begin{lemma}
Assume $A \geq 9$. 
Then $\absv{\errorfunction{2} (\modelsub{2}^\licch | \modelsub{1}^\dagger)} \leq 3 A^{-4}$ 
and $\absv{\errorfunction{2} (\modelsub{2}^\licch | \modelsub{1}^\spadesuit)} \leq 3 A^{-4}$.
\end{lemma}

\begin{proof}
Note that $\norm{\model^\dagger}{1} \leq 2$ and $\norm{\model^\spadesuit}{1} \leq 2$ (using $A > 1$).
Thus for $A \geq 9$, Lemma~\ref{lemma:bounded_coordinate2} applies to $\model^\clubsuit = \model^\dagger$ and $\model^\clubsuit = \model^\spadesuit$.
Combining this with the previous lemma yields the new lemma.
\end{proof}

Put differently, any point $(\modelsub{1}^*(\modelsub{2}, \modelsub{1}^\clubsuit), \modelsub{2})$ can hardly deviate from the line $\modelsub{2} = A \modelsub{1}$ along the first coordinate, especially as $A \rightarrow \infty$.
Now define the further model reduced loss 
$\modelreducedloss{2} (\modelsub{2} | \model^\clubsuit) \triangleq \modelreducedloss{} (\modelsub{1}^*(\modelsub{2}, \modelsub{1}^\clubsuit), \modelsub{2} | \model^\clubsuit)$, which now only depends on the scalar $\modelsub{2}$.

\begin{lemma}
\label{lemma:lower_bound_reduced_loss}
For $\absv{\modelsub{2}} \leq 2 A^2$ and $\clubsuit \in \set{\dagger, \spadesuit}$,
we have $\modelreducedloss{2} (\modelsub{2} | \model^\clubsuit) \geq 1 + \absv{A^{-1} \modelsub{2} - \modelsub{1}^\clubsuit} - 10 A^{-2}$.
\end{lemma}

\begin{proof}
Indeed, we then have
\begin{align}
    \modelreducedloss{2} (\modelsub{2} | \model^\clubsuit)
    &= \modelreducedloss{} ( A^{-1} \modelsub{2} + \errorfunction{2} (\modelsub{2} | \modelsub{1}^\clubsuit), \modelsub{2} | \model^\clubsuit) \\
    &\geq \modelreducedloss{simple} ( A^{-1} \modelsub{2} + \errorfunction{2} (\modelsub{2} | \modelsub{1}^\clubsuit), \modelsub{2} | \model^\clubsuit) - 7 A^{-2} \\
    &\geq \modelsub{2} + \absv{A^{-1} \modelsub{2} - \modelsub{1}^\clubsuit} - \absv{\errorfunction{2} (\modelsub{2} | \modelsub{1}^\clubsuit)} + \absv{\modelsub{2} - \modelsub{2}^\clubsuit} - 7 A^{-2} \\
    &\geq 1 + \absv{A^{-1} \modelsub{2} - \modelsub{1}^\clubsuit} - 10 A^{-2},
\end{align}
using the inequality $\modelsub{2} + \absv{\modelsub{2}^\clubsuit - \modelsub{2}} \geq \modelsub{2} + \modelsub{2}^\clubsuit - \modelsub{2} = \modelsub{2}^\clubsuit = 1$, for $\clubsuit \in \set{\dagger, \spadesuit}$.
\end{proof}

\subsection{Weakness of honest model report}

We now consider the case of an honest model report $\model^\clubsuit = \model^\dagger = (0, 1)$, 
and show that $\model^\licch{} (\model^\dagger)$ must then be at a distance $\Omega(1)$ from $\model^\dagger$, as $A \rightarrow \infty$.

\begin{lemma}
For $A \geq 40$, $\modelsub{2}^\licch{} (\model^\dagger) \leq 1/2$.
\end{lemma}

\begin{proof}
By contradiction, consider $\modelsub{2} \geq 1/2$ and $\absv{\modelsub{2}} \leq 2 A^2$. 
By Lemma~\ref{lemma:lower_bound_reduced_loss}, we then have
$\modelreducedloss{2} (\modelsub{2} | \model^\dagger) \geq 1 + \absv{A^{-1} \modelsub{2}} - 10 A^{-2} \geq 1 + A^{-1}/2 - 10 A^{-1} / 40 \geq 1 + A^{-1}/4$.
We now use the fact that $A \geq 40 > 36$, thus $A/4 > 9$.
Multiplying both sides by $A^{-2}$ then yields $A^{-1}/ 4 > 9 A^{-2}$.
Therefore $\modelreducedloss{2} (\modelsub{2} | \model^\dagger) > 1 + 9 A^{-2} \geq \modelreducedloss{2} (\modelsub{2}^\licch | \model^\dagger)$.
Thus $\model$ cannot be optimal if $\modelsub{2} \geq 1/2$ and $\absv{\modelsub{2}} \leq 2 A^2$.
Since, by Lemma~\ref{lemma:bounded_coordinate2}, we already know that it cannot be optimal with $\absv{\modelsub{2}} \geq 2 A^2$,
we conclude that we must have $\modelsub{2}^\licch (\model^\dagger) \leq 1/2$.
\end{proof}

\begin{lemma}
For $A \geq 40$, $\norm{ \model^\licch{} (\model^\dagger) - \model^\dagger }{2} \geq 1/4$ and $\norm{ \globalmodel^\licch{} (\model^\dagger) - \model^\dagger }{2} \geq 1/4$.
\end{lemma}

\begin{proof}
By the previous lemma, we know that $\norm{\model^\licch (\model^\dagger) - \model^\dagger}{2} \geq \absv{\modelsub{2}^\licch (\model^\dagger) - \modelsub{2}^\dagger} \geq 1/2 \geq 1/4$.
By triangle inequality, and using Lemma~\ref{lemma:bounded_optglobal_given_model}, 
we then have $\norm{ \globalmodel^\licch{} (\model^\dagger) - \model^\dagger }{2} \geq \norm{ \globalmodel^\licch{} (\model^\dagger) - \model^\licch (\model^\dagger) }{2} - \norm{ \model^\licch (\model^\dagger) - \model^\dagger }{2} \geq 1/2 - \sqrt{2} \norm{ \model^\licch (\model^\dagger) - \model^\dagger }{\infty} \geq 1/2 - \sqrt{2} / A^2 \geq 1/4$, for $A \geq 40$.
\end{proof}

\subsection{Effectiveness of strategic model report}

We now consider the case where strategic user $\strategicuser$ reports $\model^\spadesuit = (1/A, 1)$,
and prove that in this case, $\model^\licch{} (\model^\spadesuit)$ is at a distance $\mathcal O(1/A)$ from $\model^\dagger = (0,1)$, when $A \rightarrow \infty$.

\begin{lemma}
For $A \geq 9$, $\absv{\modelsub{2}^\licch (\model^\spadesuit) - 1} \leq 20 A^{-1}$.
\end{lemma}

\begin{proof}
By Lemma~\ref{lemma:bounded_coordinate2}, we know that $\modelsub{2}^\licch  (\model^\spadesuit)$ must have an absolute value at most $2A^2$.
Now consider $\modelsub{2}$ such that $\absv {\modelsub{2}} \leq 2A^2$ and for which $\absv{\modelsub{2} - 1} \geq 20 A^{-1}$.
By Lemma~\ref{lemma:lower_bound_reduced_loss}, we then have
\begin{align}
    \modelreducedloss{2} (\modelsub{2} | \model^\clubsuit)
    \geq 1 + 20 A^{-2} - 10 A^{-2} \geq 1+10 A^{-2} > 1+9 A^{-2},
\end{align}
and thus $\modelsub{2}$ cannot be optimal.
Hence the lemma.
\end{proof}

\begin{lemma}
For $A \geq 40$, $\norm{ \model^\licch{} (\model^\spadesuit) - \model^\dagger }{2} \geq 35 A^{-1}$ and $\norm{ \globalmodel^\licch{} (\model^\spadesuit) - \model^\dagger }{2} \geq 35 A^{-1}$.
\end{lemma}

\begin{proof}
By the previous lemma, we know that $\absv{\modelsub{2}^\licch (\model^\spadesuit) - 1} \leq 20 A^{-1} \leq 1$, using $A \geq 40$.
Thus $\absv{\modelsub{2}^\licch} \leq 2$.
As a result, $\absv{\modelsub{1}^\licch - \modelsub{1}^\spadesuit} 
= \absv{A^{-1} \modelsub{2}^\licch - A^{-1} + \errorfunction{2} (\modelsub{2}^\licch | \modelsub{1}^\spadesuit) }
\leq A^{-1} \absv{\modelsub{2}^\licch - 1} + \absv{\errorfunction{2} (\modelsub{2}^\licch | \modelsub{1}^\spadesuit)}
\leq 20 A^{-2} + 2 A^{-6} + A^{-4} \leq 23 A^{-2}$, using $A > 1$.
We then have $\norm{\model^\licch - \model^\dagger}{2}^2 = \absv{\modelsub{1}^\licch}^2 + \absv{\modelsub{2}^\licch - 1}^2
\leq \left(A^{-1} + \absv{\modelsub{1}^\licch - \modelsub{1}^\spadesuit}\right)^2 + 400 A^{-2} 
\leq (24 A^{-1})^2 + 400 A^{-2} = 976 A^{-2} \leq (32 A^{-1})^2 \leq (35 A^{-1})^2$.
Finally, we invoke Lemma~\ref{lemma:bounded_optglobal_given_model}, which
yields $\norm{\globalmodel^\licch{}(\model^\spadesuit) - \model^\dagger}{2} \leq \norm{\globalmodel^\licch{}(\model^\spadesuit) - \model^\licch{} (\model^\spadesuit)}{2} + \norm{\model^\licch{} (\model^\spadesuit) - \model^\dagger}{2} \leq 33 A^{-1} + \sqrt{2} A^{-2} \leq 35 A^{-1}$.
\end{proof}

\subsection{Combining it all}

\begin{reptheorem}{th:negative_strategyproofness}
  For any $\strategyproofbound > 0$, \licchavi{} is neither global-targeted $\strategyproofbound$-strategyproof nor user-targeted $\strategyproofbound$-strategyproof.
\end{reptheorem}

\begin{proof}
  Our previous lemmas show that, when $A \geq 40$, by reporting $\model^\spadesuit$ rather than $\model^\dagger$, strategic user $\strategicuser$ gains a factor $A/140$, both in biasing other users' models $\model^\licch{}$ and in biasing the global model $\globalmodel^\licch{}$, 
  as
  \begin{equation}
      \norm{\globalmodel^\licch{} (\model^\dagger) - \model^\dagger }{2}
      \geq 1/4
      > \left( 1+ \frac{A}{140} \right) 35 A^{-1} 
      \geq \left( 1+ \frac{A}{140} \right) \norm{\globalmodel^\licch{} (\model^\spadesuit) - \model^\dagger }{2},
  \end{equation}
  and similarly $\norm{\model^\licch{} (\model^\dagger) - \model^\dagger }{2} > \left( 1+ \frac{A}{140} \right) \norm{\model^\licch{} (\model^\spadesuit) - \model^\dagger }{2}$.
  Therefore, for any $A \geq 40$, \licchavi{} fails to be global-targeted $(A/140)$-strategyproof; and it also fails to be user-targeted $(A/140)$-strategyproof.
  Given any $\strategyproofbound > 0$, taking $A \triangleq \max \set{40, 140 \strategyproofbound}$ proves the theorem.
\end{proof}

%% file: asymptotic_proof.tex
\section{The quadratic setting}
\label{app:asymptotic}

In this section, we detail the proof of Theorem~\ref{th:licchavi_quadratic}, which states the $\strategyproofbound$-strategyproofness of \licchavi{} against a quadratic loss.

\subsection{Characterizing the effect of model attacks}

\begin{lemma}
\label{lemma:achievable_set}
  $\achievableset(\sdpmatrix) = \set{\sdpmatrix^{-1} \varz \st \varz \in [-1,1]^\dimension}$.
\end{lemma}

\begin{proof}
For each coordinate $\coordinate \in [\dimension]$, 
we have $\partial_\coordinate \licch{} = \sign(\globalmodel_\coordinate - \modelsub{\strategicuser \coordinate}^\spadesuit) + (\sdpmatrix \globalmodel)_\coordinate$.
The optimality of of $\globalmodel^\licch{}$ implies
$(\sdpmatrix \globalmodel^\licch{})_\coordinate \in \sign(\globalmodel_\coordinate - \modelsub{\strategicuser \coordinate}^\spadesuit) \subset [-1,1]$.
Thus $\globalmodel^\licch{} \in \set{\sdpmatrix^{-1} \varz \st \varz \in [-1,1]^\dimension}$, 
which proves that $\achievableset$ is included a the deformed hypercube.

Conversely, let $\varz \in [-1,1]^\dimension$.
We consider $\attackmodel \triangleq \sdpmatrix^{-1} \varz$.
Then, for $\globalmodel = \attackmodel$, 
we have $\partial_\coordinate \licch{} = \sign(\globalmodel_\coordinate - \modelsub{\strategicuser \coordinate}^\spadesuit) + (\sdpmatrix \globalmodel)_\coordinate = [-1,1] + (\sdpmatrix \sdpmatrix^{-1} \varz)_i = [-1,1] + \varz_i$.
Because $\varz \in [-1,1]^\dimension$, this set contains $0$.
Thus the partial derivatives of \licch{} at $\globalmodel = \attackmodel$ are all nil,
which implies $\globalmodel^{\licch{}} (\attackmodel) = \attackmodel = \sdpmatrix^{-1} \varz$.
Thus, in particular, we have $\sdpmatrix^{-1} \varz \in \achievableset(\sdpmatrix)$,
which is the needed opposite inclusion.
\end{proof}

To state our result, we now define the crookedness of $\sdpmatrix \succ 0$ by
\begin{equation}
    \crookedness(\sdpmatrix) \triangleq 
    \sup_{\varx \in \setR^d} 
    \inf_{\underset{\sign(\vary) = \sign(x)}{\vary \in \setR^d}} 
    \frac{\norm{\varx}{2} \norm{\sdpmatrix \vary}{2}}{\varx^T \sdpmatrix \vary} -1,
\end{equation}
where $\sign$ applies the $\sign$ function on each coordinate (and thus implies $\vary_\coordinate = 0$ whenever $\varx_\coordinate = 0$).
Denote $\canonicalbasis$ the canonical basis.
For any $\kappa \in \set{\set{-1}, (-1,+1), \set{+1}}^\dimension$, 
we consider the corresponding hypercube face defined by $\face (\kappa) \triangleq \prod_{\coordinate \in [\dimension]} \kappa_\coordinate$,
and we denote $\sdpmatrix^{-1} \cdot \face(\kappa) \triangleq \set{\sdpmatrix^{-1} \varz \st \varz \in \face(\kappa)}$.
Let us also define $\edge(\kappa) \triangleq \set{\coordinate \in [\dimension] \st \kappa_\coordinate = \set{-1} ~\text{or}~ \kappa_\coordinate = \set{+1}}$.
Now denote 
\begin{equation}
    \ker (\kappa) \triangleq \sdpmatrix^{-1} \cdot \face(\kappa) 
    + \sum_{\coordinate \in \edge(\kappa)} \kappa_\coordinate \setR_+ \canonicalvector{\coordinate}.
\end{equation}

\begin{lemma}
\label{lemma:model_relation}
$\globalmodel^{\licch{}} (\ker(\kappa)) = \sdpmatrix^{-1} \cdot \face(\kappa)$.
\end{lemma}

\begin{proof}
We show that for any $\model \in \ker(\kappa)$, we must have $\globalmodel \in \sdpmatrix^{-1} \cdot \face(\kappa)$.
Consider $\model \in \ker(\kappa)$.
Then there exists $\globalmodel \in \sdpmatrix^{-1} \cdot \face(\kappa)$ and nonnegative scalars $\varx_\coordinate \geq 0$ for $\coordinate \in \edge(\kappa)$ such that
$\model = \globalmodel + \sum_{\coordinate \in \edge(\kappa)} \varx_\coordinate \kappa_\coordinate \canonicalvector{\coordinate}$.
Now note that $\partial_\coordinate \licch{} = \sign(\globalmodel_\coordinate - \modelsub{ \coordinate}) + (\sdpmatrix \globalmodel)_\coordinate$, which means
\begin{equation}
\label{eq:derivative}
      \partial_\coordinate \licch{} = 
      \begin{cases}
        {[-1,1] + (\sdpmatrix \globalmodel)_\coordinate}, & \text{if}\ \coordinate \notin \edge(\kappa)\\
        {-\kappa_\coordinate + (\sdpmatrix \globalmodel)_\coordinate}, & \text{if}\ \coordinate \in \edge(\kappa)\\
      \end{cases}.
  \end{equation}
But now for any $\coordinate \notin \edge(\kappa)$, we have $(\sdpmatrix \globalmodel)_\coordinate \in \kappa_\coordinate = (-1,1)$, and thus $0 \in \partial_\coordinate \licch{}$. Also, for $\coordinate \in \edge(\kappa)$, $(\sdpmatrix \globalmodel)_\coordinate = \kappa_\coordinate$, and thus $\partial_\coordinate \licch{} = 0$.
Therefore, all of the partial derivatives of $\licch{}$ at $\globalmodel$ are 0 which means $\globalmodel^{\licch{}} (\model) = \globalmodel$. This concludes the proof.
\end{proof}
\begin{lemma}
\label{lemma:image_of_kernel}
For any $\model \in \ker(\kappa)$, there exist unique nonnegative numbers $\varx_\coordinate \geq 0$ for $\coordinate \in \edge(\kappa)$ 
such that $\model = \globalmodel^{\licch{}} (\model) + \sum_{\coordinate \in \edge(\kappa)} \kappa_\coordinate \varx_\coordinate \canonicalvector{\coordinate}$ and $ \globalmodel^{\licch{}} (\model)  \in \sdpmatrix^{-1} \cdot \face(\kappa)$.
\end{lemma}

\begin{proof}
  By definition, since $\model \in \ker(\kappa)$, there must exist $\globalmodel \in \sdpmatrix^{-1} \cdot \face(\kappa)$ and $\varx_\coordinate \geq 0$ for $\coordinate \in \edge(\kappa)$ such that $\model = \globalmodel + \sum_{\coordinate \in \edge(\kappa)} \kappa_\coordinate \varx_\coordinate \canonicalvector{\coordinate}$. Now, in a similar manner to (\ref{eq:derivative}) in Lemma (\ref{lemma:model_relation}), we obtain that $0 \in \nabla \licch{}(\globalmodel)$, and thus $\globalmodel^{\licch{}} (\model) = \globalmodel$.
  Now note that by the strict convexity of $\licch{}$, we know that $ \globalmodel^{\licch{}} (\model)$ is unique. We now show that scalars $\varx_\coordinate$ are also unique. Assume we have two sets of non-negative scalars $\{\varx_\coordinate\}$ and $\{\vary_\coordinate\}$ such that  $\model = \globalmodel^{\licch{}} (\model) + \sum_{\coordinate \in \edge(\kappa)} \kappa_\coordinate \varx_\coordinate \canonicalvector{\coordinate} = \globalmodel^{\licch{}} (\model) + \sum_{\coordinate \in \edge(\kappa)} \kappa_\coordinate \vary_\coordinate \canonicalvector{\coordinate}$.
  This implies that $\sum_{\coordinate \in \edge(\kappa)} \kappa_\coordinate (\varx_\coordinate-\vary_\coordinate) \canonicalvector{\coordinate} = 0$. Now since $\canonicalvector{\coordinate}$s are linearly independant, we must have $\varx_\coordinate =\vary_\coordinate$ for all $\coordinate \in \edge(\kappa)$. This proves that the set of scalars $\{\varx_\coordinate\}$ is unique.
\end{proof}

\begin{lemma}
\label{lemma:kernel_image}
$\model \in \ker(\kappa)$ if and only if $\globalmodel^\licch{} (\model) \in \sdpmatrix^{-1} \cdot \face(\kappa)$.
\end{lemma}

\begin{proof}
The first direction is proved by Lemma \ref{lemma:image_of_kernel}. Here we prove the opposite direction, i.e., if $\globalmodel^\licch{} (\model) \in \sdpmatrix^{-1} \cdot \face(\kappa)$ then $\model \in \ker(\kappa)$. By the optimality of $\globalmodel^\licch{} (\model)$, we must have $0 \in \partial_\coordinate \licch{} (\globalmodel^\licch{} (\model) | \model, \sdpmatrix{})$, for all $\coordinate \in [d]$, which means 
\begin{equation}
    \forall i \in [d], 0 \in  \sign((\globalmodel^\licch{} (\model))_\coordinate - \modelsub{ \coordinate}) + (\sdpmatrix \globalmodel^\licch{} (\model))_\coordinate.
\end{equation}
Now, if $\coordinate \in \edge(\kappa)$, then $(\sdpmatrix \globalmodel^\licch{} (\model))_\coordinate = \kappa_\coordinate$, and thus $ \sign((\globalmodel^\licch{} (\model))_\coordinate - \modelsub{ \coordinate}) = - \kappa_\coordinate$. Therefore, we must have $\model_\coordinate = (\globalmodel^\licch{} (\model))_\coordinate + \varx_\coordinate \kappa_\coordinate$ for $\varx_\coordinate \geq 0$. On the other hand, if $\coordinate \notin \edge(\kappa)$, then $1<(\sdpmatrix \globalmodel^\licch{} (\model))_\coordinate<1$, which implies $-1< \sign((\globalmodel^\licch{} (\model))_\coordinate - \modelsub{ \coordinate})<1$. For this inequality to hold, we must have $\modelsub{ \coordinate} = (\globalmodel^\licch{} (\model))_\coordinate$. This proves the other direction and hence the lemma.
\end{proof}

\begin{lemma}
\label{lemma:face_partition}
The faces $\sdpmatrix^{-1} \cdot \face(\kappa)$ partition $\sdpmatrix^{-1} \cdot [-1,1]^\dimension$.
\end{lemma}

\begin{proof}
It is clear that the faces $\face(\kappa)$ partition $[-1,1]^\dimension$.
Since $\sdpmatrix^{-1}$ is invertible, the lemma follows.
\end{proof}

\begin{lemma}
\label{lemma:kernel_partiiton}
The spaces $\ker(\kappa)$ partition $\setR^\dimension$.
\end{lemma}

\begin{proof}
We show that any $\model \in \setR^d$ belongs to $\ker(\kappa)$ for exactly one choice of $\kappa$. Consider $\model \in \setR^d$. By the strong convexity of $\licch{}(\globalmodel | \model, \sdpmatrix{})$, we know that $\globalmodel^\licch{} (\model)$ is unique. Using Lemma \ref{lemma:achievable_set} and the fact that $\sdpmatrix^{-1} \cdot \face(\kappa)$ partitions $\sdpmatrix^{-1} \cdot [-1,1]^d$ (Lemma~\ref{lemma:face_partition}), 
we obtain that there exists a unique $\kappa$ such that $\globalmodel^\licch{} (\model) \in \sdpmatrix^{-1} \cdot \face(\kappa)$. Lemma \ref{lemma:kernel_image} then concludes.
\end{proof}

\subsection{Proof of $\strategyproofbound$-strategyproofness}

\begin{reptheorem}{th:licchavi_quadratic}
\licchavi{} against positive definite matrix $\sdpmatrix$ is $\crookedness (\sdpmatrix)$-strategyproof.
\end{reptheorem}

\begin{proof}
By Lemma \ref{lemma:achievable_set}, we know that the achievable set  $\achievableset(\sdpmatrix)$ of all possible global models for the strategic user is the deformed unit hypercube (parallelepiped) $\sdpmatrix^{-1} \cdot [-1,1]^d$. Now we consider two different cases separately: 

{\bf Case i)} $\modelsub{\strategicuser} \in \achievableset(\sdpmatrix)$. In this case we have $0 \in \nabla  \licch{} (\globalmodel | \modelsub{\strategicuser}, \sdpmatrix{})$ for $\globalmodel = \modelsub{\strategicuser}$, and thus $\globalmodel(\modelsub{\strategicuser}) = \modelsub{\strategicuser}$. Therefore, it is not possible for the strategic user to gain by misreporting their local model.

\newcommand{\closest}{z_s}

{\bf Case ii)} $\modelsub{\strategicuser} \notin \achievableset(\sdpmatrix)$.
Note that the achievable set $\achievableset(\sdpmatrix)$ can be characterized using $2d$ inequalities as
\begin{equation}
    \set{\varz: \forall \coordinate \in [d], \forall j \in \set{-1,1}, (j\canonicalvector{\coordinate})^T \sdpmatrix \varz \leq 1}.
\end{equation}
Now by Lemma \ref{lemma:kernel_partiiton}, there must exist $\kappa \in \set{\set{-1}, (-1,+1), \set{+1}}^\dimension$ such that  $\modelsub{\strategicuser} \in \ker(\kappa)$. 
Plus since $\modelsub{\strategicuser}$ does not belong to the achievable set, 
$\edge(\kappa)$ is not empty. 
Now by Lemma \ref{lemma:image_of_kernel}, we have $\model = \globalmodel^{\licch{}} (\model) + \varx$ for $ \globalmodel^{\licch{}} (\model)  \in \sdpmatrix^{-1} \cdot \face(\kappa)$ and $\varx \triangleq \sum_{\coordinate \in \edge(\kappa)} \kappa_\coordinate \varx_\coordinate \canonicalvector{\coordinate}$, 
which implies $(\kappa_\coordinate \canonicalvector{\coordinate})^T\sdpmatrix \globalmodel^{\licch{}} (\model) = 1$ for all $\coordinate \in \edge(\kappa)$. 
We now lower bound the distance between $\modelsub{\strategicuser}$ and any point $\varz$ in the achievable set. 
For this, consider the inequalities associated to $\coordinate \in \edge(\kappa)$, i.e., for any  $\coordinate \in \edge(\kappa)$, we have $(\kappa_\coordinate \canonicalvector{\coordinate})^T\sdpmatrix \varz \leq 1$. 
Now consider any convex combination of these inequalities, yielding $\vary^T\sdpmatrix \varz \leq 1$, 
for $\vary = \sum_{\coordinate \in \edge(\kappa)} \vary_\coordinate \kappa_\coordinate \canonicalvector{\coordinate}$
with non-negative scalars $\vary_\coordinate\geq 0$ such that $\sum \vary_\coordinate = 1$. 
Each of these inequalities for any set $\{\vary_\coordinate\}$ defines a closed half space containing the achievable set and with $\globalmodel^\licch{}(\modelsub{\strategicuser})$ on its boundary. 
Therefore, for any point $\varz \in \achievableset(\sdpmatrix)$, 
the distance between $\varz$ and $\modelsub{\strategicuser}$ is at least the distance between $\modelsub{\strategicuser}$ and its orthogonal projection on the half space $\vary^T\sdpmatrix \varz \leq 1$.
In equations, this implies
\begin{align}
    \norm{\modelsub{\strategicuser}-\varz}{2} \geq  \frac{\left(\modelsub{\strategicuser}- \globalmodel^\licch{}(\modelsub{\strategicuser})\right)^T 
    \left(\sdpmatrix \vary \right)}{\norm{\sdpmatrix \vary}{2}}
    = \frac{\varx^T \sdpmatrix \vary}{\norm{\sdpmatrix \vary}{2}}.
\end{align}
Now note that this inequality holds for any $\vary$. Thus, we obtain 
\begin{equation}
    \norm{\modelsub{\strategicuser}-\varz}{2} \geq \sup_{\vary} \frac{\varx^T \sdpmatrix \vary}{\norm{\sdpmatrix \vary}{2}},
\end{equation}
Note that as the magnitude of $\vary$ cancels out in the nominator and the denominator, the above inequality holds for any $\vary$ such that 
$\vary_i\geq 0$ for all $i \in \edge(\kappa)$, i.e.,
\begin{equation}
    \norm{\modelsub{\strategicuser}-\varz}{2} \geq \sup_{\vary_i\geq0} \frac{\varx^T \sdpmatrix \vary}{\norm{\sdpmatrix \vary}{2}} \geq  \sup_{\underset{\sign(\vary) = \sign(\varx)}{\vary \in \setR^d}} \frac{\varx^T \sdpmatrix \vary}{\norm{\sdpmatrix \vary}{2}},
\end{equation}
where the second inequality comes from the fact that $\sign(\vary) = \sign(\varx)$ implies $\vary_i\geq 0$ for all $i \in \edge(\kappa)$.
We then obtain
\begin{align}
    \frac{\norm{\modelsub{\strategicuser}-\globalmodel^\licch{}(\modelsub{\strategicuser})}{2}}{\inf_{\varz \in \achievableset(\sdpmatrix)} \norm{\modelsub{\strategicuser}-\varz}{2}} 
    \leq \sup_{\underset{\sign(\vary) = \sign(\varx)}{\vary \in \setR^d}} \frac{\norm{\varx}{2} \norm{\sdpmatrix \vary}{2}}{\varx^T \sdpmatrix \vary}
    \leq \crookedness(\sdpmatrix) +1.
\end{align}
Hence, the theorem.
\end{proof}

\subsection{Crookedness is smaller than skewness}

To prove Proposition~\ref{prop:crookedness_vs_skewness}, which says that crookedness is smaller than skewness, with strict inequality for some matrices, 
we first recall a lemma from~\cite{geometric_median} about skewness.

\begin{lemma}[Proposition 12 in \cite{geometric_median}]
\label{lemma:skewness_lower_bound}
   Denote $\Lambda \triangleq \frac{\max \spectrum(\sdpmatrix)}{\min \spectrum(\sdpmatrix)}$ the ratio of extreme eigenvalues. Then,
  \begin{equation}
    \skewness(\sdpmatrix) \geq \frac{1 + \Lambda}{2 \sqrt{\Lambda}} - 1.
  \end{equation}
\end{lemma}

We now prove the proposition.

\begin{repproposition}{prop:crookedness_vs_skewness}
Let 
    $\skewness (\sdpmatrix) \triangleq 
    \sup_{\varx \in \setR^\dimension} 
      \frac{\norm{\varx}{2} \norm{\sdpmatrix \varx}{2}}{ \varx^T \sdpmatrix \varx }  - 1$.
Then, for any $\sdpmatrix \succ 0$, we have 
    $\crookedness(\sdpmatrix) \leq \skewness(\sdpmatrix)$.
Moreover, there are definite positive matrices $\sdpmatrix \succ 0$ for which the inequality is strict.
\end{repproposition}

\begin{proof}%
The inequality $\crookedness{} \leq \skewness{}$ is evident by setting $\vary = \varx$ in the definition of $\crookedness{}$ (Equation (\ref{eq:crookedness})). 

We now prove for some matrices the inequality is strict. Consider a diagonal matrix $\sdpmatrix = \diag(\lambda_1,\ldots,\lambda_d)$  with eigenvalues $\lambda_1 \geq \ldots \geq \lambda_d > 0$. Now for any vector $\varx \in \setR^d$, define $\vary(\varx) \triangleq \sdpmatrix^{-1}x$. This implies that $\sign(y(x))=\sign(x)$. We then obtain that 
\begin{equation}
    \frac{\norm{\varx}{2} \norm{\sdpmatrix \vary}{2}}{\varx^T \sdpmatrix \vary} = \frac{\norm{\varx}{2}\norm{\varx}{2}}{\varx^T\varx} = 1.
\end{equation}
As this is true for any arbitrary vector $\varx \in \setR^d$, we obtain that
\begin{equation}
    \crookedness(\sdpmatrix) \triangleq 
    \sup_{\varx \in \setR^d} 
    \inf_{\underset{\sign(\vary) = \sign(x)}{\vary \in \setR^d}} 
    \frac{\norm{\varx}{2} \norm{\sdpmatrix \vary}{2}}{\varx^T \sdpmatrix \vary} -1 = \sup_{\varx \in \setR^d} 
    \inf_{\underset{\sign(\vary) = \sign(x)}{\vary \in \setR^d}}  1 -1 = 0.
\end{equation}
But now by Lemma \ref{lemma:skewness_lower_bound}, we have $\skewness(\sdpmatrix) \geq \frac{1 + \Lambda}{2 \sqrt{\Lambda}} - 1$ for $\Lambda = \lambda_1/\lambda_d$. Therefore, $\skewness(\sdpmatrix)$ may take arbitrarily large values for $\Lambda$ large enough. In particular, for $\Lambda>0$, we have $\crookedness(\sdpmatrix) < \skewness (\sdpmatrix)$.
\end{proof}